\documentclass[lettersize,journal]{IEEEtran}  

\usepackage{graphicx} 
\usepackage{float}
\usepackage{amsthm}
\usepackage{amsmath} 
\usepackage{amssymb}  
\usepackage{wrapfig}

\usepackage{arydshln}
\usepackage{setspace}
\usepackage{xcolor}
\usepackage{cite}
\usepackage{comment}
\usepackage{hyperref}

\newtheorem{lemma}{Lemma}
\newtheorem{definition}{Definition}
\newtheorem{assumption}{Assumption}

\DeclareFontFamily{U}{mathx}{}
\DeclareFontShape{U}{mathx}{m}{n}{<-> mathx10}{}
\DeclareSymbolFont{mathx}{U}{mathx}{m}{n}
\DeclareMathAccent{\widehat}{0}{mathx}{"70}
\DeclareMathAccent{\widecheck}{0}{mathx}{"71}

\newcommand{\zc}{\widecheck{z}}
\newcommand{\ax}{a_x}
\newcommand{\ay}{a_y}
\newcommand{\az}{\widecheck{a}_z}
\newcommand{\ez}{\widecheck{e}_z}
\newcommand{\setSi}{\mathbb{S}_{(i)}}
\newcommand{\setSii}{\mathbb{S}_{(ii)}}
\newcommand{\setP}{\mathbb{P}}
\newcommand{\setF}{\mathbb{F}}
\newcommand{\setM}{\mathbb{M}}
\newcommand{\aof}{\kappa_A^O}

\newcommand*{\circledbullet}{%
    \mathbin{%
        \ooalign{$\circledcirc$\cr\hidewidth$\bullet$\hidewidth}%
    }%
}

\newtheorem{proposition}{Proposition}

\title{\LARGE \bf
Safe Aerial Manipulator Maneuvering and Force Exertion\\via Control Barrier Functions
}

\author{Dimitris Chaikalis$^{1}$, Vinicius Gon\c{c}alves$^{2}$, Nikolaos Evangeliou$^{3}$, Anthony Tzes$^{2,3}$ and Farshad Khorrami$^{1,2}$
\thanks{$^{1}$New York University, Electrical \& Computer Engineering, Brooklyn, NY 11201, USA.} %
\thanks{$^{2}$NYUAD Center for Artificial Intelligence and Robotics, UAE}
\thanks{$^{3}$New York University Abu Dhabi (NYUAD), Electrical Engineering, Abu Dhabi 129188, UAE}
\thanks{Corresponding author's email {\tt\small dimitris.chaikalis@nyu.edu}. Dimitris Chaikalis and
Vinicius Gon\c{c}alves contributed equally to this work.}
}
\begin{document}
\maketitle
\begin{abstract}
This article introduces a safe control strategy for application of forces to an external object using a dexterous robotic arm mounted on an unmanned Aerial Vehicle (UAV). A hybrid force-motion controller has been developed for this purpose. This controller employs a Control Barrier Function (CBF) constraint within an optimization framework based on Quadratic Programming (QP). The objective is to enforce a predefined relationship between the end-effector's approach motion and its alignment with the surface, thereby ensuring safe operational dynamics. No compliance model for the environment is necessary to implement the controller, provided end-effector force feedback exists. Furthermore, the paper provides formal results, like guarantees of feasibility for the optimization problem, continuity of the controller input as a function of the configuration, and Lyapunov stability. In addition, it presents experimental results in various situations to demonstrate its practical applicability on an aerial manipulator platform.
\end{abstract}
%
\section{Introduction}
%

The safe interaction of autonomous vehicles with their environment has many applications. Advances in control and technology have enabled the deployment of autonomous mobile robots for a variety of real-world tasks, often operating among other robots or humans\cite{optim_coop_arms_planning}.
Especially, the augmentation of autonomous vehicles with the necessary hardware and control methods that enable physical interaction allows their use in important tasks such as transportation\cite{multi_uav_transport_cbfs}, non-destructive testing\cite{physical_tase} and cooperation with humans\cite{dimarogonas}. This increase in capabilities in turn necessitates the development of safety-oriented control frameworks, capable of guaranteeing safe operation during autonomous physical interaction.

The typical interaction-enabled vehicle is comprised of a mobile platform (UGV, quadruped, multirotor aerial vehicle) endowed with a dexterous robotic arm. Among these, the aerial vehicle allows the highest mobility since it can navigate in the full $3$-dimensional space, resulting in an Unmanned Aerial Manipulator (UAM) when combined with a robotic arm.

Such platforms exhibit increased complexity of flight dynamics\cite{aerial_manipulation_tro_review} due to the floating base (UAV) and the robotic manipulator. However, UAMs have large dexterity in aerial manipulation\cite{aerial_manipulation_tro_review} and significant advances have been achieved on the accurate control of such systems\cite{aerial_manipulator_millimeter,aerial_manip_ctrl_uncertainty,yu2020finite,ding2022fault}, even in the context of force exchange during contact with the environment.

Irrespective of the platform, the paradigm of force exertion on an object is usually solved by utilizing a planner, tasked with driving the robot end-effector to the desired location, while a different controller is activated when in contact~\cite{BBPWPASN_tro21}. Such simple methods rely on a distinction between free-motion and motion under contact, with no safety guarantees in case of, for example, undesirable contact due to motion controller under-performance during navigation.

Current control methods mostly handle safety concerns by designing robust, prescribed-performance motion controllers, under the effect of disturbances or reaction forces from the environment\cite{yu2020finite,liang2023adaptive}. Such approaches, however, neglect the fact that simultaneously attempting position and force convergence might be undesirable, depending on the initial placement of the end-effector. Specially designed control methods need to be implemented to guarantee that contact with the environment will only occur under strictly safe conditions. With this controller, the system should be capable of avoiding undesirable contact while aligning itself with the specified goal or while recovering from a large disturbance. 

This is illustrated in Figure~\ref{fig:safeapproach}, that presents the top view of a drone and the surface in which force should be exerted. Simple controllers that only drive the alignment error to zero would not be able to safely guide the drone from the starting state (a) towards the goal state (c), since they would cause a collision with the surface. A maneuver (represented by the magenta dashed line) is necessary. The proposed controller implements such maneuvers using control barrier functions (CBFs), without the need of using path/trajectory planning.
\begin{figure}[htbp]
    \centering
    \includegraphics[width=0.9\columnwidth,keepaspectratio]{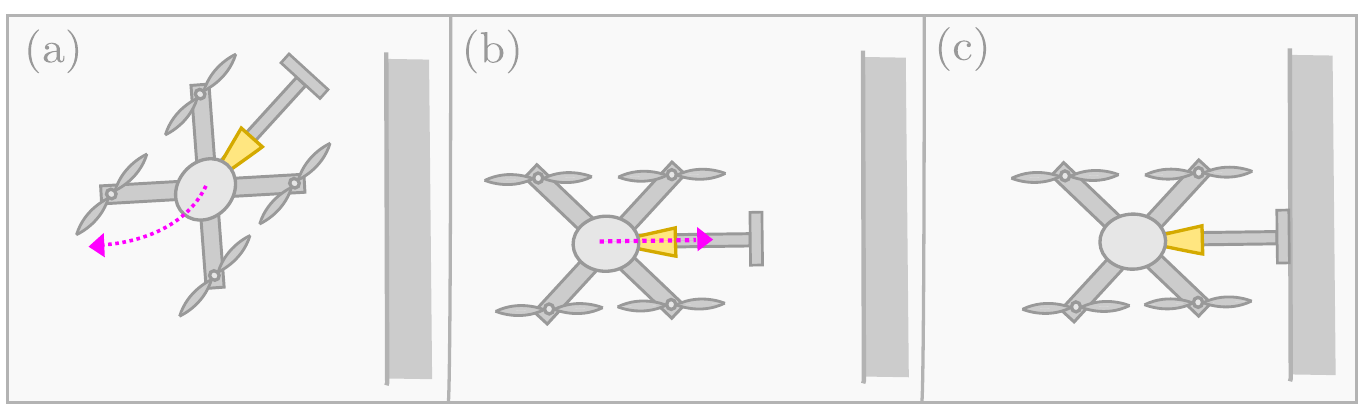}
    \caption{(a) Starting position, (b) intermediate position, and (c) final position of UAM for force exertion task.}
    \label{fig:safeapproach}
\end{figure}

This work presents a control method capable of achieving force exertion along the normal axis of arbitrarily-oriented surfaces, applied on an aerial manipulator. Compared to existing literature, increased focus is placed on ensuring system safety during task execution. The control input is computed using a Quadratic Programming (QP) approach together with CBFs\cite{8796030,gonccalves2023safe,Dai2023,gonccalves2024control,DHK23} designed to enforce a desired relation between distance to the surface and end-effector alignment, while enforcing the system's safety. 
\section{Related Work and Contribution}

Barrier functions have emerged as a powerful tool for ensuring safety in autonomous vehicle control. In \cite{cbfs_qp_safety}, a simple framework is established for safe control synthesis by incorporating CBFs in optimization-based controllers, while similar ideas are used in \cite{cbfs_tcst} to ensure safety constraints are met in the operation of an autonomous vehicle. 
 
In the specific context of hybrid force and position control, perhaps the most prominent method is admittance control. Under the admittance paradigm, reference signals are altered to prescribe mass-spring-damper end-effector motion dynamics. This is done to accommodate external wrenches, as for example in ~\cite{aerial_screwing}, where compliance is needed in order to perform aerial screwing. Similar ideas are used in \cite{smrcka2021admittance} for contact inspection, where admittance control enables a UAV to stabilize against a wall. 

Another example where contact force control is transformed into a position control problem, enforcing mass-spring-damper dynamics is \cite{physical_tase}, where accurate knowledge of all underlying models is used. In the absence of known environment model, adaptive impedance control has been applied. One example is \cite{cdc_soft_adaptive_impedance} presnting simulations of contact with a soft environment, or \cite{liang2023adaptive} with experiments on an aerial manipulator platform. Admittance methods ensure the system are able to react to measured forces, but are merely reactive and fail to address the primary reason that leads to undesirable forces, which is contact under misalignment. 

In the above works, motion planners are utilized to drive the vehicle towards the contact location. Another example is \cite{su2023sequential}, where planning is tightly coupled to a hierarchical controller, attempting to maintain errors small before proceeding to a new task. These solutions necessitate a clear distinction between free-motion and contact segments, which is not necessary in our proposed approach. Moreover, errors which can always arise from disturbances or the exchanged forces may result in switching between planning and contact phase, making it difficult to prove the stability of the overall framework. Holistic control solutions can handle the entire task of approach and force exertion, unifying the control design and removing the need of explicit planners. A typical example of such methods is model predictive control and its variations.

MPC implementations for aerial force control can be found in~ \cite{tzoumanikas2020aerial, MPC_Force_Robots}. Works featuring MPC while attempting to enforce some safety considerations include \cite{tcst_force_exertion}, where gaussian processes are incorporated to learn a relationship between motions and forces, thus enabling better predictions of generated forces. In \cite{safe_compliant_interaction}, admittance and MPC are combined for a mixture of safe compliant behavior during flight of an aerial manipulator and stiffer impedance during task execution. As opposed to MPC-based approaches, our framework doesn't require a dynamic model for the robot and the environment, while also being significantly more computationally efficient, due to the absence of the prediction horizon.

Other relevant safety-oriented approaches in force control include \cite{ecc24_cbf_force}, where CBFs are designed to enforce strict error bounds during force exertion, albeit only featuring simulation results. Likewise in \cite{barrier_backstepping_hybrid_fp_ctrl}, a backstepping controller is combined with barrier functions to ensure force limits are observed. In \cite{rong2022robust}, a fault-tolerant robust controller is applied on aerial force-motion control aiming to ensure stability under harsh conditions. Finally, omni-directional aerial vehicles is introduced, which are capable of handling a larger envelope of disturbances, to increase flight safety\cite{omni-forces-dob}. However, none of these methods are related to safety in terms of task execution explicitly, which this paper attempts to achieve.


The contributions of this work are:

\begin{itemize}
    \item A CBF-based controller that implements safe force exertion without the need of explicitly using path/trajectory planners. This controller is based on a novel inequality constraint that enforces a prescribed relationship between distance and alignment of the end-effector and the environment. This eliminates the unsafe situation of contact under undesirable conditions. The controller is lightweight and is based on strictly convex QPs, which can be solved extremely efficiently. Furthermore, it is very flexible, since the controller has several parameters that can be chosen according to the limitations/needs of each platform/situation. Finally, it is simple to implement, requiring only a QP solver.

    \item Several formal results for the safe motion/force controller are provided in Section \ref{sec:formalg} to address: safety guarantees, continuity, feasibility, and Lyapunov stability.

    \item Several experimental results on a two degree-of-freedom arm mounted on a quadrotor are provided. This shows the efficacy of the proposed controller under a wide range of scenarios.
\end{itemize}


This paper is structured as follows. Section \ref{sec:matnot} presents the mathematical notation. Section~\ref{sec:psa} presents the problem statement and the assumptions. Section~\ref{sec:high-level} outlines the overall proposed optimization-based controller, with its mathematical properties and equilibrium analysis in Section~\ref{sec:formalg}. Section \ref{sec:exp} presents experimental results and Section \ref{sec:concl} the conclusion. Finally, all proofs of Propositions and Lemmas are given in Appendices A and B, respectively. 

\section{Mathematical notation}
\label{sec:matnot}

The symbol $\mathbb{R}^+$ denotes the set of nonnegative reals. All vectors in this paper are column vectors. If $A$ is a matrix, $A^\top$ represents its transpose. The notation $\dot{f} = \frac{df}{dt}$ is used for time derivatives. $\nabla_g$ represents the gradient, as a column vector, taken in the variable-$g$. If $f: \mathbb{R}^a \mapsto \mathbb{R}^b$ is a function of the variable $q$, $\frac{\partial f}{\partial q}$ is its \emph{Jacobian} $a \times b$ matrix in which the entry in the row $i$ and column $j$ is $\frac{\partial f_i}{\partial q_j}$. $\| \cdot \|$ is the Euclidean norm. For a scalar $h > 0$, $[x]^h \triangleq x|x|^{h-1}$ is the signed power function. $R_x(\phi), R_y(\theta) , R_z(\psi)$ represent the $3 \times 3$ rotation matrices in the $x, y$ and $z$ axis, respectively. Transformations (translation, rotation or both) \emph{from} frame $\mathcal{A}$ \emph{to} frame $\mathcal{B}$ are written using the notation $(\cdot)_{\mathcal{A}}^{\mathcal{B}}$ (e.g., $R_{\mathcal{A}}^{\mathcal{B}}$ for a rotation). If $u$ and $v$ are vectors, $u>v, u<v, u \leq v, u \geq v$ are inequalities that should be taken componentwise. 

Normalized 3D vectors (axis) have the symbol $\widecheck{(\cdot)}$ (e.g., $\zc, \ez, \widecheck{e}_x$, etc...). The vectors ${\widecheck{e}_x}, {\widecheck{e}_y}, {\ez}$ are the columns of the $3 \times 3$ identity matrix. Sets are denoted using blackboard bold (e.g., $\mathbb{P}, \mathbb{F}, \mathbb{M}$, etc.) and frames using calligraphic font (e.g., $\mathcal{W}$, $\mathcal{U}$, $\mathcal{E}$, etc.). As in the standard pictorial physics notation, in a drawing, a symbol like $\otimes$ means a vector that is directed \emph{inside} the image, whereas $\circledbullet$ means a vector that is directed \emph{outside} the image.

A function $V: \mathbb{R}^n \mapsto \mathbb{R}^+$ is said to be \emph{Lyapunov-like} if it is differentiable, nonnegative and both $V$ and its gradient vanishes if and only if when the argument is the zero vector. A function $f: [a,b] \mapsto \mathbb{R}$ is said to be \emph{$\kappa$-like} \footnote{It is a generalization of the concept of $\kappa$ functions in control, but with the domain extended to any interval of the reals instead of only nonnegative intervals and also only being non-decreasing instead of increasing.} if it is continuous, non-decreasing and $f(\xi)=0$ if and only if $\xi=0$. It is said to be \emph{strictly $\kappa$-like} if it is $\kappa$-like and additionally strictly increasing and differentiable. 

\section{Problem Statement \& Assumptions }
\label{sec:psa}
The studied UAM-system comprises an underactuated multirotor aerial vehicle coupled with a robotic arm featuring revolute joints, as shown in Figure~\ref{fig:simulator}. There exists a stationary planar surface upon which it is desired to exert normal forces.
\begin{figure}[htbp]
    \centering
    \includegraphics[width=0.85\columnwidth,keepaspectratio]{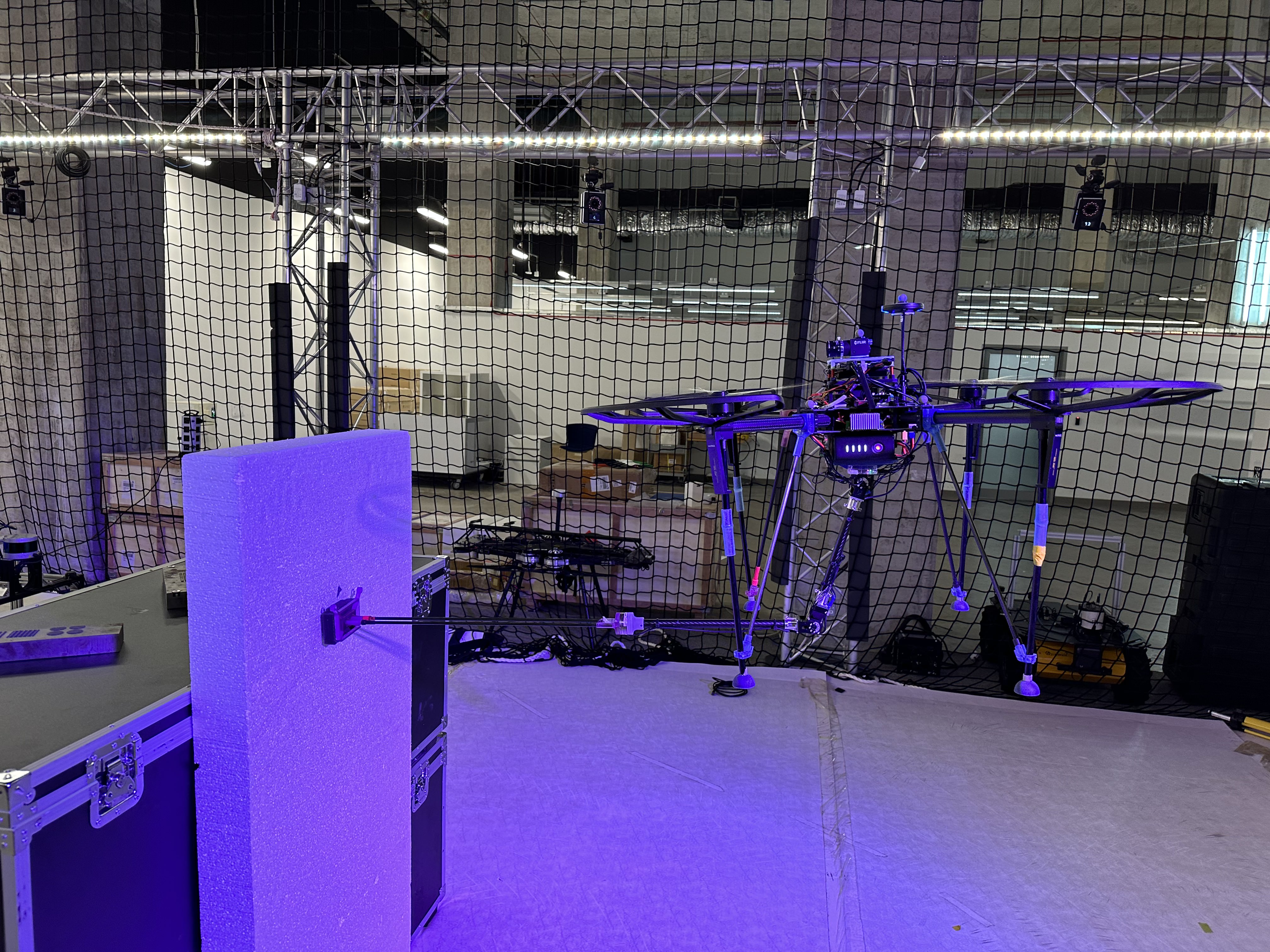}
    \caption{Example of  Unmanned Aerial Manipulator (UAM) system
    .}
    \label{fig:simulator}
\end{figure}

The relevant frames, as shown in Figure~\ref{fig:system_and_frames} for the studied problem are: $\mathcal{U}$, the frame fixed on the center of the aerial vehicle; $\mathcal{E}$, the frame of the force sensor attached to the robotic arm end-effector; $\mathcal{W}$, the world-fixed inertial frame; and $\mathcal{P}$, the inertial frame attached to the plane.
\begin{figure}[htbp]
    \centering
    \includegraphics[width=0.85\columnwidth,keepaspectratio]{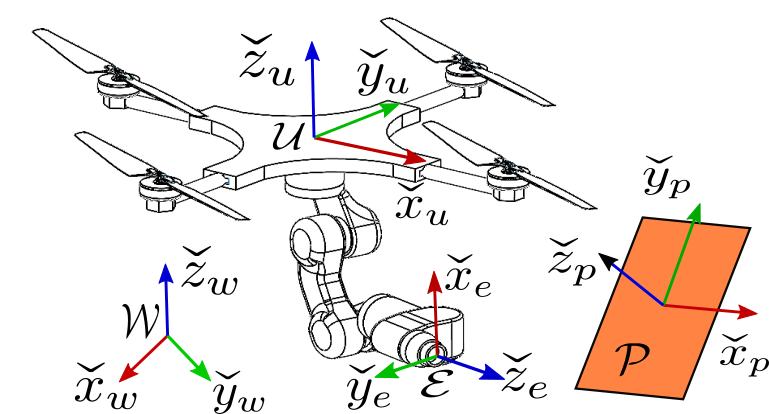}
    \caption{Coordinate frames of an example UAM system}
    \label{fig:system_and_frames}
\end{figure}

\begin{definition} \label{def:basic}
\vspace{5pt} \ \ \ 
\begin{itemize}
    \item Let ${p_{\mathcal{W}}^{\mathcal{U}}} \in \mathbb{R}^3$ denote the position of the center of the aerial vehicle measured in $\mathcal{W}$;
    \item Let $\phi$ (roll),  $\theta$ (pitch) and  $\psi$ (yaw) represent the angles of the attitude of $\mathcal{U}$ in $\mathcal{W}$. More precisely, the rotation between $\mathcal{W}$ and $\mathcal{U}$ will be given by $R_{\mathcal{W}}^{\mathcal{U}} \triangleq R_z(\psi) R_y(\theta) R_x(\phi)$;
    \item Let ${q_m} \in \mathbb{R}^n$ be the configuration of the robotic arm, represented by its joint angles;
    \item Let $\left({p_{\mathcal{P}}^{\mathcal{W}}}, {R_{\mathcal{P}}^{\mathcal{W}}} \right)$ be the constant, rigid transformation between $\mathcal{P}$ and $\mathcal{W}$. $\qedsymbol$
\end{itemize}
\end{definition}

 For clarity in mathematical exposition, the desired task is divided into two components: the \emph{alignment} and the \emph{force exertion} subtasks. More precisely, the objective is for the end-effector's position and orientation to align properly with \(\mathcal{P}\) (the \emph{alignment subtask}) and subsequently exert a constant force \( F_d < 0 \) along the  \(\zc\) axis of \(\mathcal{P}\), (the \emph{force subtask}) \footnote{Since the force $F_d$ is measured in the \(\zc\) axis, it is negative since we want it to go in the negative direction of this axis.}. 

\begin{assumption} \label{assump:plane} 
\vspace{5pt} \ \ \ 
\begin{itemize}
    \item It is assumed that the plane is static, infinite and that the $\zc$ vector of frame $\mathcal{P}$ is normal to the plane and directed outside of it (that is, pointing to the half-space that contains the UAM);
    \item It is assumed that the fixed transformation $\left({p_{\mathcal{P}}^{\mathcal{W}}}, {R_{\mathcal{P}}^{\mathcal{W}}} \right)$ is known;
    \item It is assumed that the end-effector tool that will perform the force-exertion is aligned with the $\zc$ vector of the frame $\mathcal{E}$ (see Figure~\ref{fig:system_and_frames}). 
    $\qedsymbol$
\end{itemize}  
\end{assumption}

\begin{assumption} \label{assump:force} Let the position of the end-effector along the $\zc$ axis of the plane  frame be $Z$.  The  reaction force model as a function of the insertion into the plane, $Z$, is assumed \footnote{$Z<0$ implies an insertion of the tool into the plane, whereas $Z>0$ means that no contact between the plane  and the tool was established.} to be a function $F : \mathbb{R} \rightarrow \mathbb{R}$ such that for any $F_d < 0$, there exists a $Z_d < 0$ in which $f(\xi) = F(\xi+Z_d) - F_d$ is $\kappa$-like. $\qedsymbol$
\end{assumption}

Note that the assumption for $F$ implies that there exists an unique $Z_d < 0$ such that $F(Z_d) = F_d$. A typical example satisfying the above assumptions is the spring-action force model $F(Z) = \min(k Z, 0)$ for a $k > 0$, in which $Z_d = F_d/k$. Although this assumption does not encompass all types of force models (like the Kelvin-Voigt model, in which the force depends on $\dot{Z}$ as well), it is still general enough to be used in many applications. It turns out that no explicit knowledge about neither the force model $F$ nor $Z_d$ will be necessary to implement the proposed controller. However,

\begin{assumption} \label{assump:Zd} It is assumed that a value $Z_d^*$, reasonably close to $Z_d$ and such that  $Z_d^* \leq Z_d$ is known. $\qedsymbol$
\end{assumption}

This estimate $Z_d^*$ for $Z_d$ will be useful in the controller. When the surface is rigid, $Z_d$ is very close to zero and thus $Z_d = 0$ is an excellent approximation for any reasonable force $F_d$. Consequently, any negative small number $Z_d^*$ will do. However, it will be shown later that it is beneficial to have $Z_d^*$ as closer as possible to the real $Z_d$.

Concerning the control inputs:
\begin{assumption}
\vspace{5pt} \ \ \ 
\begin{itemize}
    \item It is assumed that the manipulator is controlled through joint velocities. 
    Furthermore, it is assumed that the manipulator only has revolute and prismatic joints;
    \item It is assumed that the aerial vehicle is controlled through its linear velocity and yaw rate. $\qedsymbol$
\end{itemize}   
\end{assumption}

This work assumes a \emph{kinematic model} in which the manipulators' joint velocities and the aerial vehicle's linear velocity and yaw rate can be controlled. These will be sent to the low level controllers of the manipulator and aerial vehicle. For the latter, these will modify the vehicle's roll  and pitch in addition to the yaw to achieve the target linear velocity. Related to this, the following is assumed:

\begin{assumption} \label{assum:slowvar} It is assumed that roll ($\phi$) and pitch ($\theta$) of the aerial vehicle vary slowly, resulting in $\dot{\phi} \approx 0$, $\dot{\theta} \approx 0$. $\qedsymbol$
\end{assumption}

Assumption \ref{assum:slowvar} implies that the yaw axis (thrust direction) $\widecheck{z}_{\mathcal{U}}$ of the UAV is approximately constant. With this, the following assumption can be made:

\begin{assumption} \label{assum:nevalign} It is assumed that the yaw axis of the UAV $\widecheck{z}_{\mathcal{U}}$ is constant. Furthermore, this constant vector is not aligned with the (constant) $z$ axis of $\mathcal{P}$, $\widecheck{z}_\mathcal{P}$. $\qedsymbol$
\end{assumption}

This condition is a technicality, since $\widecheck{z}_{\mathcal{U}}$ and $\widecheck{z}_\mathcal{P}$ can be arbitrarily close to be aligned. It will, however, make the mathematical analysis simpler.

%
\section{ Controller synthesis\label{sec:high-level}}
The controller outputs linear velocity and yaw angular rate commands for the aerial vehicle as well as joint angle velocity commands for the robotic arm. In \ref{subs:basics}, basic parameters and variables are established. The Lyapunov functions for force control and alignment control are presented in \ref{subs:forceex} and \ref{subs:daed}, respectively. Velocity and joint limits are addressed in \ref{subs:jcvl} with the final controller optimization problem presented in \ref{subs:final_ctrl}.
\subsection{{Basic definitions}}
\label{subs:basics}
For the sake of simplicity, the variables in this section are expressed in the plane frame $\mathcal{P}$, and thus subscripts and superscripts (as $(\cdot)_{\mathcal{P}}^{\mathcal{U}}$) will be avoided to keep the notation as simple as possible. 
\begin{definition}
Let $x,y,z$ be the vehicle positions in $\mathcal{P}$ (center of $\mathcal{U}$ measured in $\mathcal{P}$). The  \emph{configuration vector} is defined as
$  q \triangleq [x \  y \   z \   \psi \   q_m^{\top}]^{\top}. $
The dimension of this vector will be $n$, so, it is considered that the manipulator has $n-4$ degrees of freedom. Furthermore, individual entries of $q_m$ will be denoted by $q_{m,j}$, $j=1,2,...,n-4$.
$\qedsymbol$
\end{definition}

The output command of the controller will be $\dot{q}$ and the high-level controller assumes the dynamics $\dot{q} = u$, Note that the UAV's roll ($\phi$) and pitch ($\theta$) are not included in the configuration. This is because for the proposed controller, they are assumed to be approximately constant (see Assumption \ref{assum:slowvar}), and so they are considered parameters. Furthermore

\begin{definition} Let  $X, Y, Z$ be  the coordinates of the position of the end-effector frame $\mathcal{E}$ measured in $\mathcal{P}$. Furthermore, let $\widecheck{x}, \widecheck{y}, \zc$ be the orthonormal vectors of the orientation of the end-effector frame  $\mathcal{E}$ measured in $\mathcal{P}$.  $\qedsymbol$
\end{definition}

These elements can be obtained from $q$ using forward kinematics. Another important definition related to the task is the following.

\begin{definition} \label{def:taskcoord} Define the \emph{task coordinates}  $r_X \triangleq X$, $r_Y \triangleq Y$, $r_O \triangleq  1+\ez^\top \zc$, $r_Z \triangleq Z-Z_d$.
$\qedsymbol$
\end{definition}

These coordinates codify the task mathematically as $r_X=r_Y=r_O=0$ \footnote{Note that since $\ez$ and $\zc$ are normalized vectors, $e_O = 1+\ez^\top \zc = 0$ if and only if $\zc = -\ez$, which is precisely what we require from the orientation part of the alignment.}(alignment subtask) and $r_Z=0$ (force exertion subtask) and will be useful to describe the controller and prove its mathematical properties.

\subsection{The force exertion subtask}
\label{subs:forceex}

The controller synthesis starts with the force exertion subtask and a relevant Lyapunov function is introduced.

\begin{definition} \label{def:potential} Let $\kappa_F: \mathbb{R} \times \mathbb{R} \mapsto \mathbb{R}$ be a function $\kappa_F(s_1,s_2)$ such that for any fixed $s_1$, $\kappa_F(s_1,\cdot)$ is $\kappa$-like. The function $V_F: \mathbb{R} \mapsto \mathbb{R}^+$ is defined as:
\begin{equation}
\label{eq:VF}
    V_F(r_Z) \triangleq \int_{0}^{r_Z} \kappa_F \Big(\xi+Z_d, F(\xi+Z_d) - F_d \Big)d\xi 
\end{equation}
\noindent in which $Z_d$ and $F$, the desired insertion and force-insertion relationship respectively, were introduced in Assumption \ref{assump:force}. $\qedsymbol$
    
\end{definition}

\begin{lemma} \label{lemma:VFlyap} The function $V_F$ is Lyapunov-like. $\qedsymbol$
\end{lemma}

Considering that $V_F$ is Lyapunov-like, the force-exertion subtask can be described as either $r_Z=0$ or $V_F(r_Z)=0$. 

As it will be clear soon, the Lyapunov function in \eqref{eq:VF} was crafted with three goals in mind: (i) to permit the formal proof of mathematical properties of the controller, such as Lyapunov stability, (ii) to allow flexibility in choosing the controller parameters, and (iii) to guarantee that the controller can be implemented \emph{without explicit knowledge} of neither the force-insertion model $F$ nor $Z_d$. Indeed, only the measured force $F(Z(q))$ and an estimate $Z_d^*$ of $Z_d$ will be required. Furthermore, it will not be actually necessary to compute $V_F$ to implement the controller: it is only a mathematical construct that allows the formal results to be established. 
\subsection{The alignment subtask}
\label{subs:daed}
Given the force subtask Lyapunov function $V_F$, one could compute $\dot{q}$ through a resolved rate controller scheme: if $\dot{q}$ is chosen so $\frac{d}{dt} r_Z(q) = \nabla_q r_Z(q)^\top \dot{q} = -\frac{\partial V_F}{\partial r_Z}(r_Z)$, then Lyapunov theory implies that $r_Z$ converges to $0$ asymptotically since $\dot{V}_F = -(\frac{\partial V_F}{\partial r_Z}(r_Z))^2$ and the right-hand side is nonpositive and zero if and only if $r_Z=0$. The controller $\dot{q}$ could be any solution to the linear equation $\nabla_q r_Z(q)^\top \dot{q} = -\frac{\partial V_F}{\partial r_Z}(r_Z)$.

However, this may not comply with the \emph{alignment subtask}, and the UAM may try to exert force without the proper alignment, which may compromise the safety of the system (e.g., the manipulator can collide with the plane sideways). For the purpose of addressing this issue, we need to define the \emph{alignment error function}.

\begin{definition} \label{def:A} Let 
\begin{itemize}
    \item $V_A^{XY}: \mathbb{R}^2 \mapsto \mathbb{R}^+$ be a twice differentiable Lyapunov-like function such that Hessian at the point $(0,0)$ is positive definite; 
    \item Let  $\aof: \mathbb{R} \mapsto \mathbb{R}^+$ be a  twice differentiable strictly $\kappa$-like function;
\end{itemize} 
Define the \emph{alignment error} function as 
\begin{equation}
    A \triangleq V_A^{XY}(r_X, r_Y) + \aof(r_O). \label{A_def}
\end{equation}
\end{definition}
Note that $A$ is a proper error alignment metric. Indeed, due to its properties, $A \geq 0$ and $A=0$ iff $r_X=r_Y=r_O = 0$, that is, when the alignment task is solved. Then, to ensure this safety, we need to consider a distance function between the drone and the plane. In this work, $Z(q)$, the distance between the end-effector and the plane, is selected.

With these definitions, a \emph{distance-alignment error} relationship will be enforced. Essentially, the distance between the end-effector and the plane  $Z(q)$ and the alignment $A(q)$ should follow a relationship in which there is a minimal distance between the end-effector and the surface, $Z$, for each alignment error $A$. If the force subtask is solved, the distance decreases so the end-effector can have the insertion necessary to provide the desired force, $Z=Z_d$, and thus this distance-alignment error relationship will force the alignment task variables $r_X$, $r_Y$ and $r_O$ to be close to $0$.

Geometrically, this can be visualized by plotting $A$ and $Z$ in a diagram and enforcing the points along the trajectory to be above a certain curve. This curve should be such that when $Z=Z_d$, the maximum allowed alignment error should be small. This maximum allowed error should be ideally zero, but  that requires the perfect knowledge of the value of $Z_d$. This implies that the closer is $Z_d^*$ to the true $Z_d$, the smaller is the error that we can enforce. To formally define this alignment error-distance relationship, we introduce the following definition:

\begin{definition} \label{def:kappa} Let 
\begin{itemize}
    \item $\kappa_{A}: \mathbb{R}^+ \mapsto \mathbb{R}^+$  be a strictly $\kappa$-like function;
    \item $Z_d^*$ be any number such that $Z_d^* < Z_d$ (see Assumption \ref{assump:Zd}). 
\end{itemize} 
Define the  \emph{barrier distance-alignment relation} as
\begin{equation}
    B \triangleq Z - Z_d^* - \kappa_{A}(A).\label{B_def}   \ \ \ \qedsymbol
\end{equation}
\end{definition}

Given the definition of $B$, the desired relationship can be described by $B(q) \geq 0$. Note that it prescribes a minimum distance $Z$ for each alignment error $A$. This relationship is parametrized by the function $\kappa_A$.
Figure~\ref{fig:diagram} displays this diagram with $\kappa_{A}(s) = s/(\sqrt{s}+0.2)$, $Z_d=-0.1$ and $Z_d^{*}=-0.3$. When $Z=Z_d$, a maximum alignment of $\kappa_A^{-1}(Z_d-Z_d^*) \approx 0.1$ is enforced.

\begin{figure}[htbp]
    \centering
    \includegraphics[width=0.9\columnwidth,keepaspectratio]{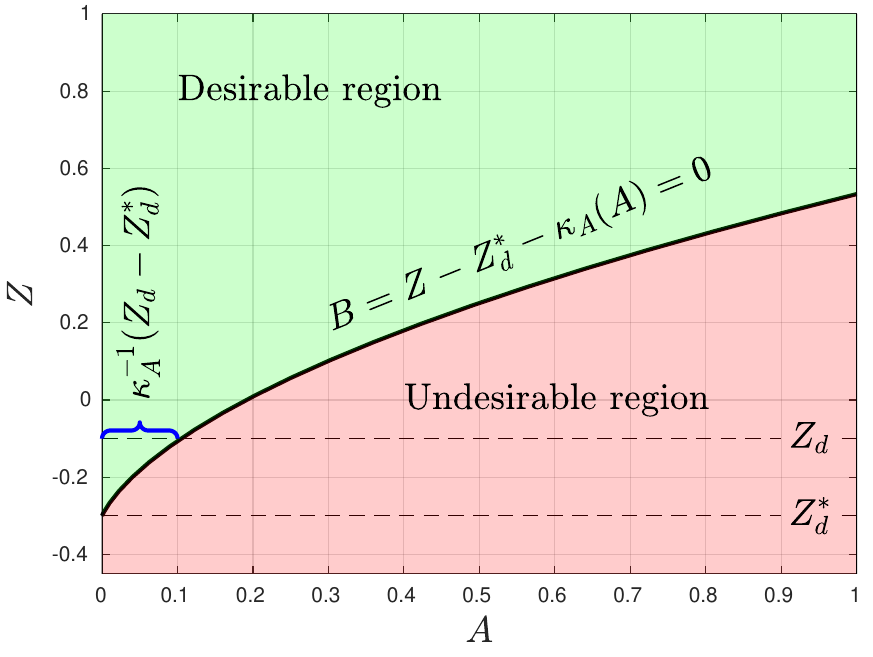}
    \caption{Alignment error ($A$) vs. distance ($Z$) function diagram for UAM. Any safe trajectory $q(t)$ should be such that $(A(q(t)),Z(q(t)))$ should be in the green region (i.e., $B(q(t)) > 0$). This will be enforced by the controller.}
    \label{fig:diagram}
\end{figure}

Note that Definition \ref{def:kappa}  is such that the point $Z=Z_d$, $A=0$ is above the barrier curve of \eqref{B_def} ($B \geq 0$) in the $(A,Z)$ diagram. This is necessary because otherwise the operational point in which the desired force is exerted would not be complying with the safety constraint $B \geq 0$. Regardless, it is sufficient to select any $Z_d^* < Z_d$. This is convenient because, in practice, it is difficult to know exactly the value of $Z_d$. The value of $Z_d^*$ , however, should be chosen judiciously: selecting a very negative number will make the condition $B \geq 0$ very loose and no reasonable distance-alignment relationship will be enforced. In fact, it is clear from the diagram and the properties of the function $\kappa_A$ that when $Z=Z_d$, $A$ is forced to be less or equal than $\kappa_A^{-1}(Z_d-Z_d^*)$. Consequently, the smaller is $Z_d-Z_d^*$ and/or the steeper is the function $\kappa_{A}$ around $0$, the smaller is the enforced steady state alignment error. See Figure  \ref{fig:diagram}.

Since $r_X, r_Y, r_O$ are functions of the configuration $q$, so is $A(q)$. Furthermore, the distance $Z$ is clearly a function of $q$ as well, and consequently so is $B = B(q)$.  Finally, note that \eqref{A_def} has a separated term for the $r_X, r_Y$ part of the subtask ($V_A^{XY}$) and one for the $r_O$ part of the subtask ($\aof$). This separation is not merely structural; it plays a vital role in establishing certain mathematical properties of the controller.

As opposed to the force subtask, that will be enforced by trying to enforce a resolved rate controller as much as possible, the alignment subtask will be enforced through an inequality constraint in an optimization problem. The reasons for that are twofold. First, it is desirable to impose not only a steady-state constraint  that the alignment should be zero, but a \emph{hard-constraint} ($B \geq 0$) on the behavior of the alignment during the whole trajectory. For this purpose, a hard constraint in the optimization problem for computing the control input is more suitable than trying to enforce a resolved-rate for the alignment error using the objective function. This is because in the latter there are no guarantees that the resolved-rate will be achieved. Second, as it was described before, a \emph{minimal} alignment is imposed to each distance, and thus the task is better described as an inequality instead of an equality like in the resolved rate approach.
Since $B$ is differentiable, the constraint $B \geq 0$ will  be enforced using a CBF inequality.

\subsection{{Joint and configuration velocity limits}}
\label{subs:jcvl}
Limits for $\dot{q}$ and joint limits for the manipulator will also be enforced. Consider joint limits for the manipulator $\underline{q}_m \leq q_m \leq \overline{q}_m$ and configuration velocity limits. The joint limits can be enforced through CBFs, and the corresponding differential inequalities can be merged into the configuration velocity limits. For this, the following definition is necessary.

\begin{definition} \label{def:lim} Split $\underline{u}$ and $\overline{u}$ in its first four entries  $\underline{u}_D$, $\overline{u}_D$ and remaining entries $\underline{u}_m$, $\overline{u}_m$. Let $\underline{q}_m < \overline{q}_m$, $\underline{u}_D < 0 <  \overline{u}_D$,  $\underline{u}_m < 0 < \overline{u}_m$. Let $K_L$ be a positive scalar.

Define
\begin{eqnarray}
\label{eq:bdef}
 && \underline{b}(q) \triangleq [\underline{u}_D^{\top} \ \ \max(\underline{u}_m, K_{L}(\underline{q}_m-q_m))^{\top}]^{\top} \nonumber\\
&& \overline{b}(q) \triangleq [\overline{u}_D^{\top} \ \ \min(\overline{u}_m, K_{L}(\overline{q}_m-q_m))^{\top}]^{\top}   \label{q_m_cbf}
\end{eqnarray}
\noindent in which the maximum and minimum are taken component-wise. $\qedsymbol$
\end{definition}
Then  $\underline{b}(q) \leq \dot{q} \leq \overline{b} (q)$ implements the  constraint. 
\subsection{{Optimization-based controller}}
\label{subs:final_ctrl}
The optimization problem for computing the control inputs is defined in this section. The idea is to have the controller be as close as possible to a resolved rate controller for the force subtask (Subsection \ref{subs:forceex}), while considering the alignment relationship (Subsection \ref{subs:daed}) for the alignment subtask. Additionally, joint/configuration velocity limits are taken into account (Subsection \ref{subs:jcvl}). Before presenting the controller, it is necessary to introduce the following definition.

\begin{definition} \label{def:W} Let $\epsilon_i$, $i=1,2,4,...,n$, be positive scalars and $\epsilon_3=0$. Define $E$ as the diagonal matrix with entries $\epsilon_i$. Then, define $W: \mathbb{R}^n \times \mathbb{R}^n \mapsto \mathbb{R}^+$ as 
\begin{equation}
W(q,\dot{q}) \triangleq \left(\nabla_q r_Z(q)^\top\dot{q} {+} \frac{\partial V_F}{\partial r_Z}\big(r_Z(q)\big)\right)^2 {+} \dot{q}^{\top}E\dot{q}. \ \ \ \qedsymbol
\label{W_def}
\end{equation}    
\end{definition}

Note that from \eqref{W_def}, $W(q,\dot{q})$ is quadratic and convex in the variable $\dot{q}$ for all fixed $q$. In fact, it will be shown that it is \emph{strictly convex} in $\dot{q}$ for all fixed $q$. The existence of the term $\dot{q}^\top E \dot{q}$ is to ensure this strict convexity. The reason why $\epsilon_3 = 0$ (associated with $\dot{q}_3=\dot{z}$) instead of being positive as the remaining $\epsilon_i$ is twofold: a) it is not necessary to ensure strict convexity, and b) makes the mathematical analysis simpler.

The proposed controller to solve the complete (alignment+force exertion) task can then be defined using the aforementioned elements.

\begin{definition} \label{def:hlvlu}  Define the vector $b^*(q) \in \mathbb{R}^n$ such that the $i_{th}$ entry is $\underline{b}_i(q)$ if $\frac{\partial B}{\partial q_i}(q)$ is negative  and $\overline{b}_i(q)$ otherwise.  Let $\kappa_B: \mathbb{R} \mapsto \mathbb{R}$ be a $\kappa$-like function and
\begin{eqnarray}
\label{eq:dotqmax}
    && \setF \triangleq \{q \in \mathbb{R}^n \ | \  \nabla_q B(q) ^{\top}b^*(q) \geq -\kappa_B(B(q))\}.
\end{eqnarray}
The proposed control function $u_H: \setF \mapsto \mathbb{R}^n$ will be defined as the solution of the following QP:
\begin{eqnarray}
\label{eq:uH}
    u_H(q) &\triangleq& \textsl{arg}\min_\mu \ \ \  W(q,\mu) ~ \mbox{such that} \nonumber \\
        && \nabla_q B(q)^{\top} \mu \geq -\kappa_B \big( B(q) \big) \nonumber \\
        && \underline{b}(q) \leq \mu \leq \overline{b} (q). \ \ \ \qedsymbol
\end{eqnarray}
\end{definition}

It is implicitly assumed that the optimization problem is feasible and has a unique solution. These aspects will be discussed in Section \ref{sec:formalg}, but, overall, the optimization problem is feasible if and only if $q \in \setF$ (that is why the domain of $u_H$ is $\setF$). Furtermore, the solution is always unique when it exists. Nevertheless, the objective function tries to enforce the resolved rate controller relationship $\frac{d}{dt} r_Z(q) = -\frac{\partial V_F}{\partial r_Z}(r_Z)$ as much as possible while considering the constraints. The term $\dot{q}^{\top}E\dot{q}$ is a regularization term, and it is very important because otherwise $W$ is not strictly convex on the decision variable $\dot{q}=\mu$. Strict convexity, as it will be discussed in Section \ref{sec:formalg}, is necessary to guarantee uniqueness and continuity of the solution $u_H(q)$. 

It is important to stress $u_H$ is indeed a function only of $q$, thereby rendering the closed-loop dynamical system $\dot{q} = u_H(q)$ autonomous. This is because, due to Assumption \ref{assump:force},  the reaction force is modeled to be configuration-dependent through the function $F(Z(q))$. However, it should be noted that \emph{it is not necessary} to know the force-reaction model function $F$  to compute this controller. Indeed,  these terms only appear in $\frac{\partial}{\partial r_z} V_F\big(r_Z(q)\big)$. But, using the integral definition of $V_F$ in Definition \ref{def:potential} and the fundamental theorem of calculus:
\[
    \frac{\partial V_F}{\partial r_Z}(r_Z) = \kappa_F\big(r_Z{+}Z_d,F(r_Z{+}Z_d){-}F_d\big) = \kappa_F\big(Z,F-F_d\big).
\]
Note that this is a function only of $Z, F$ and $F_d$. Consequently, 
this term can be implemented solely by knowing $F(Z)$. This quantity can be obtained without any information about the force model provided that there is a force sensor that measures it. The mention to the force model is only necessary for the sake of the mathematical analysis.

\section{Controller Analysis \label{sec:formalg}}

Some formal guarantees and justifications about the proposed controller design of Equation \eqref{eq:uH} will be established. 

First, it is necessary to consider the following result.

\begin{lemma}\label{lemma:pdresulta} The following results hold $\forall q$.
\begin{eqnarray}
\hspace*{-14mm}
    && \frac{\partial X(q)}{\partial x} = 1 \ , \ \frac{\partial Y(q)}{\partial x} = 0 \ , \ \frac{\partial  \zc(q)}{\partial x} = 0 \ , \ \frac{\partial Z(q)}{\partial x} = 0 \nonumber \\
\hspace*{-14mm}
    && \frac{\partial X(q)}{\partial y} = 0 \ , \ \frac{\partial Y(q)}{\partial y} = 1 \ , \ \frac{\partial  \zc(q)}{\partial y} = 0 \ , \ \frac{\partial Z(q)}{\partial y} = 0 \nonumber \\
\hspace*{-14mm}
    && \frac{\partial X(q)}{\partial z} = 0 \ , \ \frac{\partial Y(q)}{\partial z} = 0 \ , \ \frac{\partial  \zc(q)}{\partial z} = 0 \ , \ \frac{\partial Z(q)}{\partial z} = 1 \nonumber ~\qedsymbol
\end{eqnarray}
\end{lemma}

\subsection{Feasibility, uniqueness and continuity of the controller}
The first key result pertaining to the controller is:

\begin{proposition} \label{prop:nonzerograd} $\nabla_q B(q)$ is non-zero for all $q$. In particular $\frac{\partial B}{\partial z}(q)=1$ for all $q$. $\qedsymbol$
\end{proposition}

This proposition guarantees that the left-hand side of the CBF constraint is always non-null. The intuitive explanation to this result is very simple: it is always possible to improve the distance-alignment relationship (i.e., increase $B$) by moving the vehicle in the direction of $z$, that is, moving away from the plane.

The following definition will be important for the subsequent results.
\begin{definition} Let
\[
    \setP\triangleq\left\{q  \in \mathbb{R}^n \  \lvert \  B(q) > 0 \ , \ \underline{q}_m < q_m < \overline{q}_m \right\}.
\qedsymbol   
\]
\end{definition}

\begin{proposition} \label{prop:pi}  {The set $\setP$ is positive invariant.}$\qedsymbol$
\end{proposition}

The next two propositions address the feasibility of the QP in \eqref{eq:uH}.
\begin{proposition}
    \label{prop:feasF} The QP in \eqref{eq:uH} is feasible if and only if $q \in \setF$. $\qedsymbol$
\end{proposition}

\begin{proposition} \label{prop:zero} {If the initial configuration $q_0 \in \setP$, then the optimization problem is always feasible. In particular, $\mu=0$ is always a strictly feasible solution.} $\qedsymbol$
\end{proposition}

The intuition for the above result is that ``stopping'' ($\dot{q} = \mu = 0$) is always a feasible action considering the constraints if the UAM is inside the safe region $\setP$. This is obvious because this will not change anything in the system. 

Note that the combination of Proposition \ref{prop:zero} and Proposition \ref{prop:feasF}  implies that $\setP \subseteq \setF$. This is because in $\setP$, the QP is feasible (Proposition \ref{prop:zero}), and this is exactly what characterizes $\setF$ (Proposition \ref{prop:feasF}). However, there may exist points of $\setF$ outside of $\setP$: those are the points that do not comply with the imposed constraints ($B(q) \geq 0$ and joint limits) but in which $u_H(q)$ is still well defined. This is highly beneficial, because it may be the case that, due to a starting condition (or due to disturbances), the system may fall outside of $\setP$, and in that case the controller will drive the system back towards it.

Uniqueness and continuity of the solution will be addressed in the Propositions \ref{prop:unique} and \ref{eq:propcontinuity}. However, before delving into these, it is essential to introduce the following lemma:

\begin{lemma} \label{lemma:pod} For all $ q$, $\|\dot{r}_Z\|^2 + \dot{q}^\top E \dot{q} = \dot{q}^\top ( \nabla_q r_Z(q) \nabla_q r_Z^{\top}(q) + E)\dot{q}$ is nonnegative and zero if and only if $\dot{q}=0$. Equivalently, the matrix $\nabla_q r_Z(q) \nabla_q r_Z^{\top}(q)+E$ is positive definite, and $W(q,\dot{q})$ (as in \eqref{W_def}) is strictly convex in $\dot{q}$ for all fixed $q$. 
$\qedsymbol$ \end{lemma}


\begin{proposition} \label{prop:unique} {The solution to the optimization problem in \eqref{eq:uH} is unique $\forall q \in \setF$.   $\qedsymbol$}
\end{proposition}

\begin{proposition} \label{eq:propcontinuity} {For $q \in \setP$, $u_H(q)$ is a continuous function
. $\qedsymbol$}
\end{proposition}

\subsection{Stability results and equilibrium point analysis}
The results in Subsections B and C deal with the Lyapunov stability and equilibrium point analysis of the closed-loop system $\dot{q} = u_H(q)$.

\begin{proposition}
   \label{prop:lyapst} {For all $q \in \setP$, the closed loop system using \eqref{eq:uH}, $\dot{q} = u_H(q)$ is such that $\Dot{V}_F = \nabla V_F(q)^{\top} u_H(q)\leq 0$. Furthermore, $\dot{V}_F = 0$ if and only if $\dot{q} = u_H(q)=0$. $\qedsymbol$} 
\end{proposition}

Proposition \ref{prop:lyapst} establishes Lyapunov stability for the \emph{force exertion subtask}. The fact that $\dot{V}_F = 0$ if and only if $\dot{q} = 0$ also precludes ``stable motions'' in which the force task stays constant ($\dot{V}_F=0$) but the system is moving ($\dot{q}\not=0$). Lyapunov stability for the \emph{alignment subtask} is slightly trickier to obtain, and comes under a mild assumption.  

\begin{proposition} \label{prop:Adot}
    Assume no input limits for $\dot{q}_3 = \dot{z}$. For all $q \in \setP$, the closed loop system using \eqref{eq:uH}, $\dot{q} = u_H(q)$ is such that $\dot{A} = \nabla A(q)^{\top} u_H(q)\leq 0$. Furthermore, $\dot{A} = 0$ if and only if $E\dot{q} =Eu_H(q)=0$, i.e., when all the entries of $\dot{q}$, except maybe $\dot{q}_3 = \dot{z}$, are $0$. $\qedsymbol$
\end{proposition}

Proposition \ref{prop:Adot} implies that, assuming that $\dot{z}$ is unlimited, then the alignment subtask Lyapunov function $A$ is also non-increasing. Note that the assumption regarding $\dot{z}$ is relatively mild. The bounds for $\dot{z}$ can be effectively imposed by appropriately shaping the function $\kappa_F$. This is achieved by ensuring that the function does not necessitate an excessively rapid decrease in the distance $Z$, which is primarily influenced by $z$.

Assuming no input limits for $\dot{z}=\dot{q}_3$ and that $q_0 \in \setP$, Propositions \ref{prop:lyapst} and \ref{prop:Adot} together imply an interesting behavior for the proposed controller. Consider the $(A,Z)$ diagram in Figure \ref{fig:diagram}. According to Proposition \ref{prop:pi}, the set $\setP$ is positive invariant, and thus $q(t) \in \setP$ for all $t \geq 0$ and Propositions \ref{prop:lyapst} and \ref{prop:Adot} can be used in tandem. Note that $\dot{V}_F \leq 0$, concluded from Proposition \ref{prop:lyapst}, implies that $\dot{Z} \leq 0$ as long as $Z \geq Z_d$. Thus, considering the system configuration $q$ projected as a point $(A(q),Z(q))$ in this diagram, moving as the configuration changes, then, as long as the point is above the line $Z = Z_d$ , the controller will \emph{never} make the point move either up or right in the diagram. Furthermore, as long as the system is moving in the configuration space, $\dot{q} \not = 0$, then $\dot{V}_F < 0 \rightarrow \dot{Z} < 0$, and thus the point is moving downward in the diagram. Also, if $E\dot{q} \not = 0$, that is, the system in the configuration space is moving in \emph{any configuration that is not $\dot{z}$}, then $\dot{A} < 0$, and the point should be moving to the left in the diagram. 


Propositions \ref{prop:lyapst} and \ref{prop:Adot} together imply a form of Lyapunov stability for the force exertion subtask and alignment subtask.  However, it does not rule out the presence of stable equilibrium points in the configuration space that does not solve the task. The equilibrium set $\{q \in \mathbb{R}^n \ | \ u_H(q)=0\}$ will be characterized using some definitions and intermediate results.

\begin{definition}\label{def:xyz}
Let:
\begin{itemize}
    \item $\az$ denote the normal vector to the UAM (the direction in which the yaw is applied) measured in $\mathcal{P}$;
    \item $\ay \triangleq \ez \times \az$ and $\ax \triangleq \ay \times \az$;
    \item $\widecheck{n}_j(q)$ denotes the normalized rotation axis vector of the $j^{th}$  joint of the manipulator, measured in $\mathcal{P}$ and as a function of $q$, if it is a rotational joint, and the zero 3D vector if it is a prismatic joint. $\qedsymbol$
\end{itemize}

\end{definition}

Note that Assumption \ref{assum:nevalign} implies that $\ay \not=0$, $\ax \not=0$. Furthermore, that Assumption \ref{assum:slowvar} implies that the vector $\az$ is constant. The vectors $\ax, \ay, \az$ form an orthogonal (but not necessarily orthonormal, since the vectors $\ax, \ay$ are not necessarily normalized and hence why they do not have the $\widecheck{(\cdot)}$ symbol)  triple of vectors that will be important in the forthcoming analysis. 

\begin{lemma} \label{lemma:partialz} 
It holds that, $\forall q$ and manipulator joints $j$
\begin{equation}
\label{eq:roteqs}
  \frac{\partial}{\partial \psi} \zc(q) = \az \times \zc(q) \ , \  \frac{\partial}{\partial q_{m,j}} \zc(q) = \widecheck{n}_j(q) \times \zc(q)
\end{equation}
\noindent Furthermore:
\begin{eqnarray}
\label{eq:roteqs2}
    \hspace*{-0.5cm}
    && \hspace*{-0.5cm} \frac{\partial}{\partial \psi} r_O(q) {=} \ez^{\top}(\az {\times} \zc(q)) {=} (\ez {\times} \az)^\top \zc(q) {=} \ay^\top \zc(q); \nonumber \\
    \hspace*{-1cm}
    && \hspace*{-1cm} \frac{\partial}{\partial q_{m,j}} r_O(q) = \ez^\top (\widecheck{n}_j(q) \times \zc(q)) = (\ez \times \widecheck{n}_j(q))^{\top}\zc(q).  \nonumber 
\end{eqnarray}$\qedsymbol$
\end{lemma}

\begin{lemma} \label{lemma:skkt} $u_H(q)=0$ if and only if $q$ is such that there exists $\underline{\lambda}_{m,j}$, $\overline{\lambda}_{m,j}$, $j=1,2,...,n-4$ that make the following KKT-like conditions hold true:
\begin{eqnarray}
\label{eq:skkt}
    && \mbox{Stationarity}:  \nonumber \\
&& \kappa_F \frac{\partial }{\partial x} \kappa_{A}(A) = 0 \ , \    \kappa_F \frac{\partial }{\partial y} \kappa_{A}(A) = 0 \ \nonumber \\
&&  \kappa_F \frac{\partial }{\partial \psi} \kappa_{A}(A) = 0. \nonumber \\
&&  \kappa_F \frac{\partial }{\partial q_{m,j}} \kappa_{A}(A) =  \underline{\lambda}_{m,j}{-}\overline{\lambda}_{m,j} \ \mbox{for $j=1,..,n{-}4$} \nonumber \\
    && \mbox{Complementarity}:  \nonumber \\
    &&  -\kappa_B(B)\kappa_F = 0 \ , \   \nonumber \\
    &&  \ \underline{b}_{j,m} \underline{\lambda}_{j,m} = 0 \ , \  \overline{b}_{j,m} \overline{\lambda}_{j,m} = 0 \ \mbox{for $j=1,...,n{-}4$}\nonumber \\
    && \mbox{Primal feasibility}:  \nonumber \\
    && 0 \geq -\kappa_B (B) \ , \ \underline{q}_{m,j} \leq q_{m,j} \leq \overline{q}_{m,j}  \ \mbox{for $j=1,..,n{-}4$} \nonumber \\
    && \mbox{Dual feasibility}:  \nonumber \\
    && \kappa_F \geq 0 \ , \ \underline{\lambda}_{m,j} \geq 0 \ , \ \overline{\lambda}_{m,j} \geq 0 \  \mbox{for $j=1,..,n{-}4$}
\end{eqnarray}
\noindent in which $\underline{b}_{j,m}$, $\overline{b}_{j,m}$ are the components of $\underline{b}$ and $\overline{b}$ respective to the manipulator variables, and the dependencies of $\kappa_F, A$ and $B$ on $q$ are omitted. $\qedsymbol$
\end{lemma}

\begin{definition} \label{def:twosets} The following subsets of $q \in \mathbb{R}^n$ are defined:
\begin{itemize}
    \item (i) : The near-success  set $\setSi$:
    \begin{itemize}
        \item (a): \ $Z(q) = Z_d$ (or, equivalently, $F(q) = F_d$);
        \item (b): \ $A(q) \leq \kappa_A^{-1}(Z_d-Z_d^*)$;
        \item (c): \ $\underline{q}_m \leq q_m \leq \overline{q}_m$.
    \end{itemize}
    
    \item (ii) : The spurious equilibrium set $\setSii$:
    \begin{itemize}
        \item (a): \ $Z(q) > Z_d$ (or, equivalently, $F(q) = 0 \not= F_d$);
        \item (b): \ $r_X(q)=r_Y(q)=0$;
        \item (c): \ $\ay^\top \zc(q)=0$, i.e., the vectors  $\ez$, $\zc(q)$ and $\az$  are coplanar;
        \item (d): \ $B(q) = 0$;
        \item (e): \ For all manipulator joints $j$, $(\ez \times \widecheck{n}_j(q))^{\top}\zc(q)$ is nonnegative if $q_{m,j}$ is saturated below ($q_{m,j} = \underline{q}_{m,j}$), nonpositive if it is saturated above ($q_{m,j} = \overline{q}_{m,j}$) and $0$ if it is non-saturated; 
        \item (f): \ $\underline{q}_m \leq q_m \leq \overline{q}_m$. $\qedsymbol$
    \end{itemize} 
\end{itemize} 
\end{definition}
Note that the sets $\setSi$ and $\setSii$ are disjoint, $\setSi \cap \setSii = \emptyset$, since (i)-(a) and (ii)-(a) are mutually exclusive conditions. Regardless, with these definitions and results it is possible to characterize the set of equilibrium points of the dynamical system $\dot{q} = u_H(q)$.

\begin{proposition} \label{prop:zeroset}  $u_H(q) = 0$ \emph{if and only if} $q \in \setSi \cup \setSii$. $\qedsymbol$
\end{proposition}

Proposition \ref{prop:zeroset} states that equilibrium points of the controller occur \emph{if and only if} one of two mutually-exclusive conditions is met. The first condition, denoted as \(q \in \setSi\), occurs when the task has been achieved within an alignment error tolerance. Note that the terminal alignment error in this case, \( A \leq \kappa_A^{-1}(Z_d-Z_d^*) \), can be improved in two ways: by bringing \( Z_d^* \) closer to \( Z_d \) or by making the function \( \kappa_A \) steeper near 0.

The second condition, \(q \in \setSii\), applies when there is no contact with the surface. In this case, the \(X\) and \(Y\) alignments are achieved, the vectors \(\zc, \ez\), and \(\az\) are coplanar, the system lies at the boundary of the CBF constraint, and there is a specific sign condition for the joints of the manipulator, which depends on the saturation of the joints. The sign condition for the manipulator joints when $q \in \setSii$ may seem enigmatic, but it essentially indicates that for the configuration \( q \) to be a spurious (undesirable) equilibrium point, specific conditions related to gradient descent must be met. Specifically, when attempting to improve the alignment error \( r_O \) from \( q \) by applying gradient descent in the joint coordinates \( q_{m,j} \) only, one discovers that improvement is not possible. This occurs because, for each joint \( j \), either that joint cannot contribute (\( \frac{\partial r_O}{\partial q_{m,j}}(q) = (\ez \times \widecheck{n}_j)^{\top}\zc=0 \)), or the direction specified by \( -\frac{\partial r_O}{\partial q_{m,j}}(q) \) is obstructed by saturation. For instance, if \( -\frac{\partial r_O}{\partial q_{j,m}}(q)>0 \) (meaning that increasing \( q_{j,m} \) is necessary to reduce \( r_O \)), then \( q_{j,m} \) must already be at its upper saturation limit.

\section{Experimental Results}
\label{sec:exp}
In order to validate the performance of the proposed control method, experiments were carried out on a custom-built UAM. The platform is comprised of a quadcopter UAV and a $2$-DoF planar robot manipulator attached underneath it, seen in Figure~\ref{fig:experiment_pics}. The UAV weighs $\approx 3$ Kg and is equipped with the Pixhawk autopilot running ArduCopter and an Intel NUC i7 computer, while the arm joints are Dynamixel XM-430 motors. The optimization controller is computed on-board at $60$ Hz and the velocity commands are transmitted to the ArduCopter autopilot and the joint motors respectively. The actual time to compute a controller step was clocked at $2\mbox{msec}$ maximum. A motion capture system is used to provide both the autopilot and the proposed controller with position and orientation estimates for the UAV. An ATI Mini-40 FT sensor exists at the end-effector for force measurements. A video of the simulations is available at \href{https://youtu.be/BAkfiErkPYI}{this link} \footnote{https://youtu.be/BAkfiErkPYI}.

\begin{figure}[htbp]
    \centering
    \includegraphics[width=0.9\columnwidth,keepaspectratio]{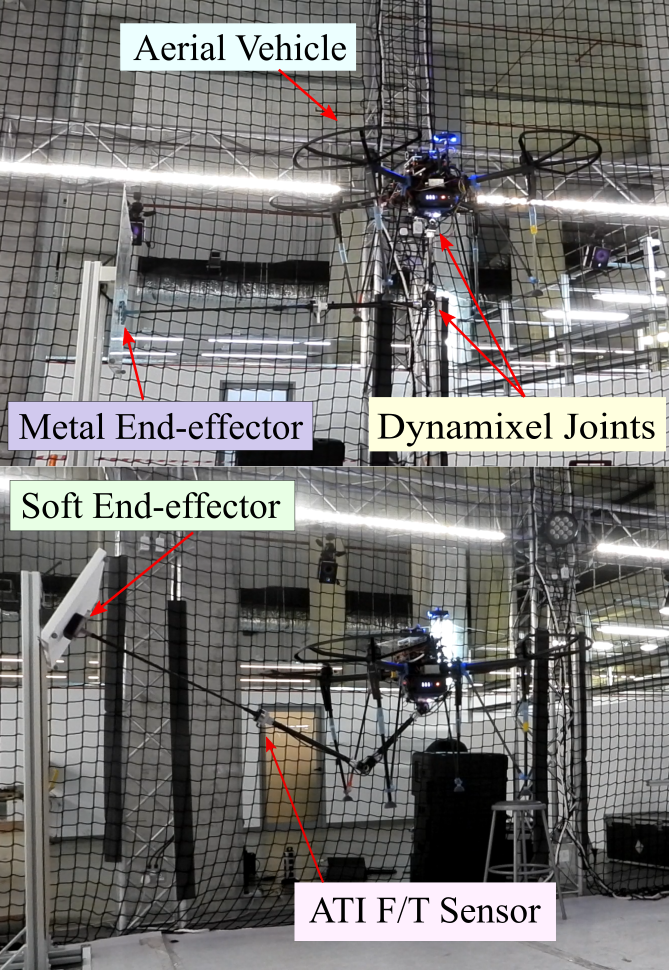}
    \caption{UAM during hard contact (top) and inclined force exertion (bottom), with annotated components.}
    \label{fig:experiment_pics}
\end{figure}

The controller parameters are the same in all presented experiments, as follows: 
\begin{eqnarray}
&& \kappa_F(s_1,s_2) =(0.12|s_1|+0.02)[s_2]^{0.5} \nonumber \\
&& \kappa_A(s) = 2.08s/(\sqrt{s}+0.29)^2 \ , \  \kappa_B(s)=-0.3s \nonumber \\ 
&& V_{A}^{XY}(r_X,r_Y) = 6.5\left(r_X^2 + r_Y^2\right) \ , \ \kappa_A^O(r_O) = 4r_O \nonumber \\
&& \underline{u}_m=-\overline{u}_m=[20 \ 20]^\top \ , \ \underline{q}_m=-\overline{q}_m=[70 \  105]^\top \nonumber \\
&& \underline{u}_D=-\overline{u}_D=[0.1  \ 0.15 \ \infty \ 5.7]^\top \ , \ K_L=0.5 \nonumber \\
&&   \epsilon_1 = \epsilon_2 = 0.04, \epsilon_4 = 4\times 10^{-5}, \epsilon_5=\epsilon_6 =  3 \times 10^{-6} \nonumber  
\end{eqnarray}
Above, all the positions are measured in meters, angles in degrees, times in seconds and forces in Newtons. The value for $F_d$ changed in each experiment, and $Z_d^*$ was selected as a very small negative number ($-0.001$m), since in all the experiments the necessary deformation to exert the force is negligible.

A rationale for selecting \(\kappa_A\) and \(\kappa_F\) will be provided. For \(\kappa_A\), a function with a steep—yet finite—inclination at \(s=0\) is desirable in order to enforce a small alignment error. Furthermore, it is preferable for \(\kappa_A\) to saturate as \(s \rightarrow \infty\). Indeed, enforcing an arbitrarily high distance when the alignment error is arbitrarily high is not necessary. The proposed function is a simple function with these characteristics.

For \(\kappa_F(s_1,s_2)\), the simplest choice might be \(\kappa_F(s_1,s_2) = K_F s_2\) for some \(K_F > 0\). However, experimental observations indicated that a nonlinear controller would be more appropriate. This is because a high value of \(K_F\) is desirable when \(s_2=F-F_d\) is close to 0, but a lower value is preferable in other cases. The choice \(\kappa_F(s_1,s_2) = K_F'[s_2]^{0.5}\) (in which $[\cdot]^n$ was defined in Section \ref{sec:matnot}) for another \(K_F' > 0\) possesses this characteristic. It is important to note that since this choice is independent of \(s_1\), during the entire no-contact phase, \(s_2 = F - F_d = -F_d\), which is constant. Therefore, the velocity driving the end-effector towards the surface remains constant until contact is made. This aspect complicates the tuning process: if \(K_F'\) is too low, the approach is excessively slow; conversely, if \(K_F'\) is too high, the force exerted is overly aggressive. Introducing a distance-dependent \(\kappa_F\) (i.e., dependent on \(s_1 = Z\)) resolves this issue. This rationale underpins the crafting of \(\kappa_F\) in Definition \ref{def:potential}, enabling this capability.



\subsection{Experiment 1: \textit{Above the $B=0$ curve}}
For the first experiment, the controller is enabled after the UAM acquires an initial pose far from the surface and significantly misaligned with it. This corresponds to an initial configuration far 'above' the proposed distance-alignment barrier curve. The target surface is vertical with the world, meaning the force exertion axis is perpendicular to gravity, with a target force of $F_d = -3N$.

The alignment errors, the applied force and the controller velocities appear in Figures~\ref{fig:cbf_above}, \ref{fig:force_above} and \ref{fig:vels_above}. 
Figure~\ref{fig:cbf_above} indicates that the controller drives the system towards the surface, while properly aligning the end-effector in order to satisfy the desired barrier function as possible, without need for a separate planner. On the left all individual alignment error metrics are plotted separately. 
It  should be noted that the ``orientation error'' is the angle computed by $(180/\pi)\cos^{-1}(1-r_O)$.
In Figure~\ref{fig:force_above}, the controller also managed to maintain the measured force close to the desired one, while the commands generated by the controller during the experiment are shown in Figure~\ref{fig:vels_above}.
\begin{figure}[htbp]
    \centering
    \includegraphics[width=0.9\columnwidth,keepaspectratio]{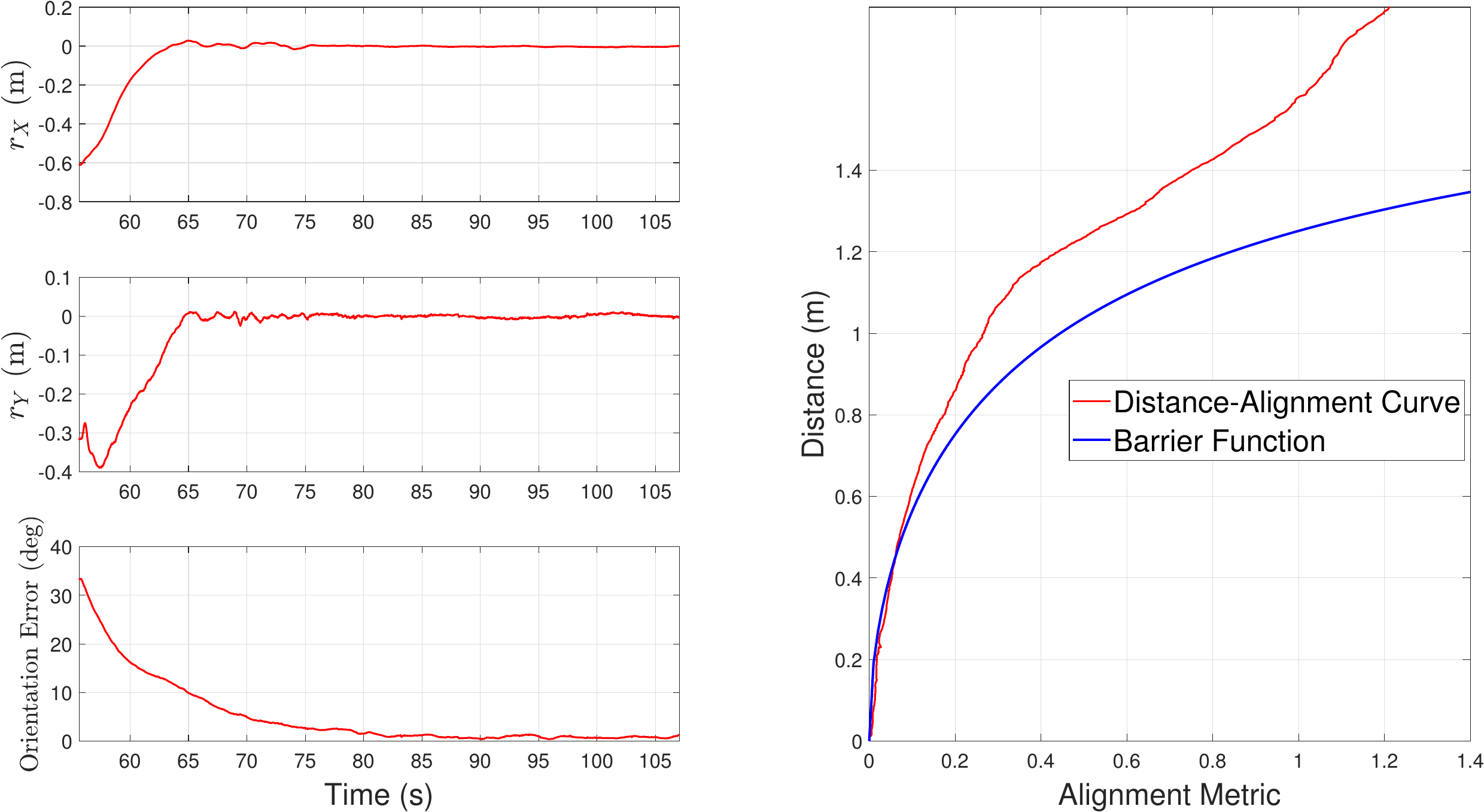}
    \caption{(Left) End-effector alignment errors (Right) Distance-alignment curve and the barrier, (Experiment 1). }
    \label{fig:cbf_above}
\end{figure}

\begin{figure}[htbp]
    \centering
    \includegraphics[width=0.7\columnwidth,keepaspectratio]{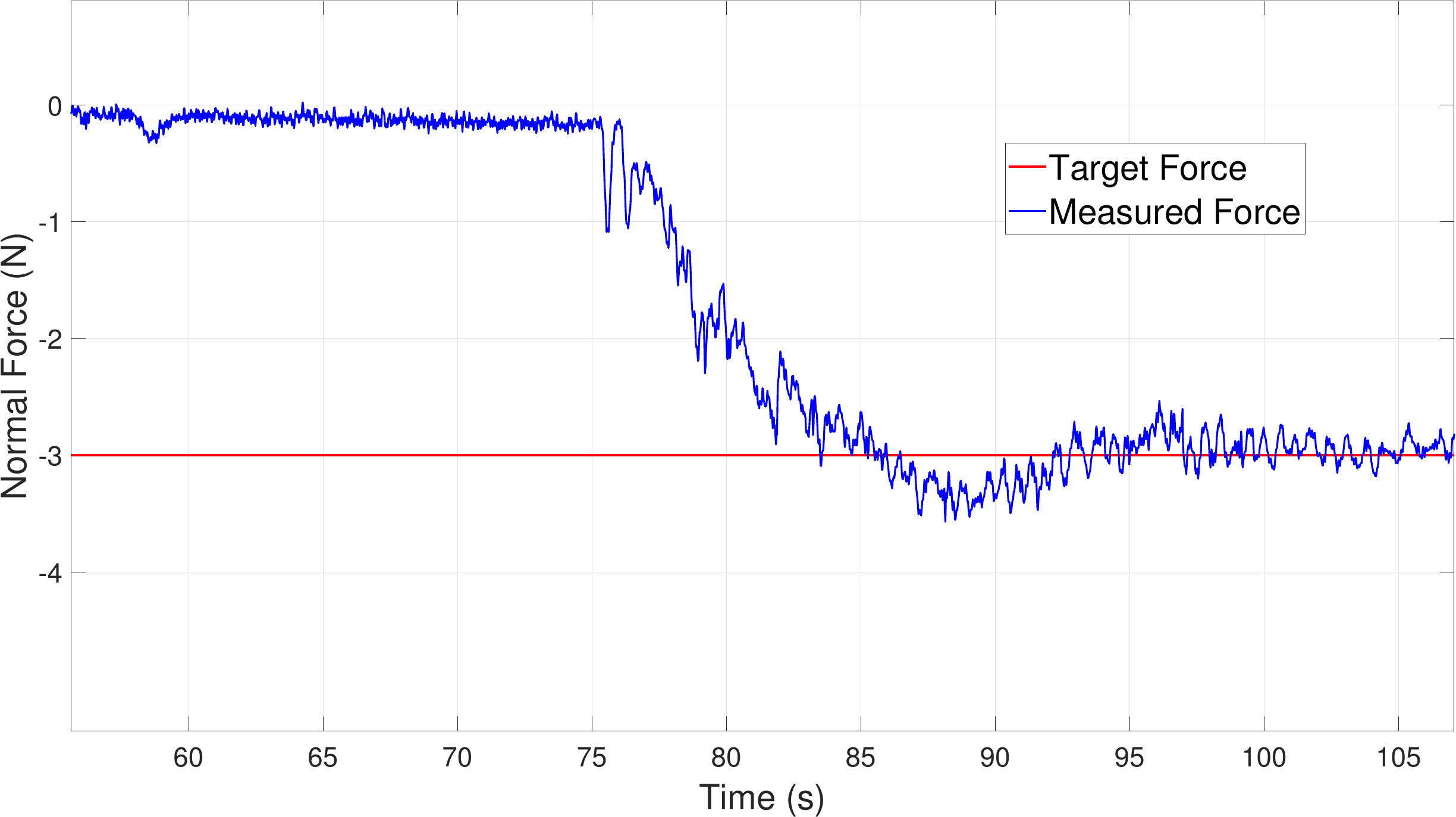}
    \caption{Measured normal force, (Experiment 1).}
    \label{fig:force_above}
\end{figure}

\begin{figure}[htbp]
    \centering
    \includegraphics[width=1\columnwidth,keepaspectratio]{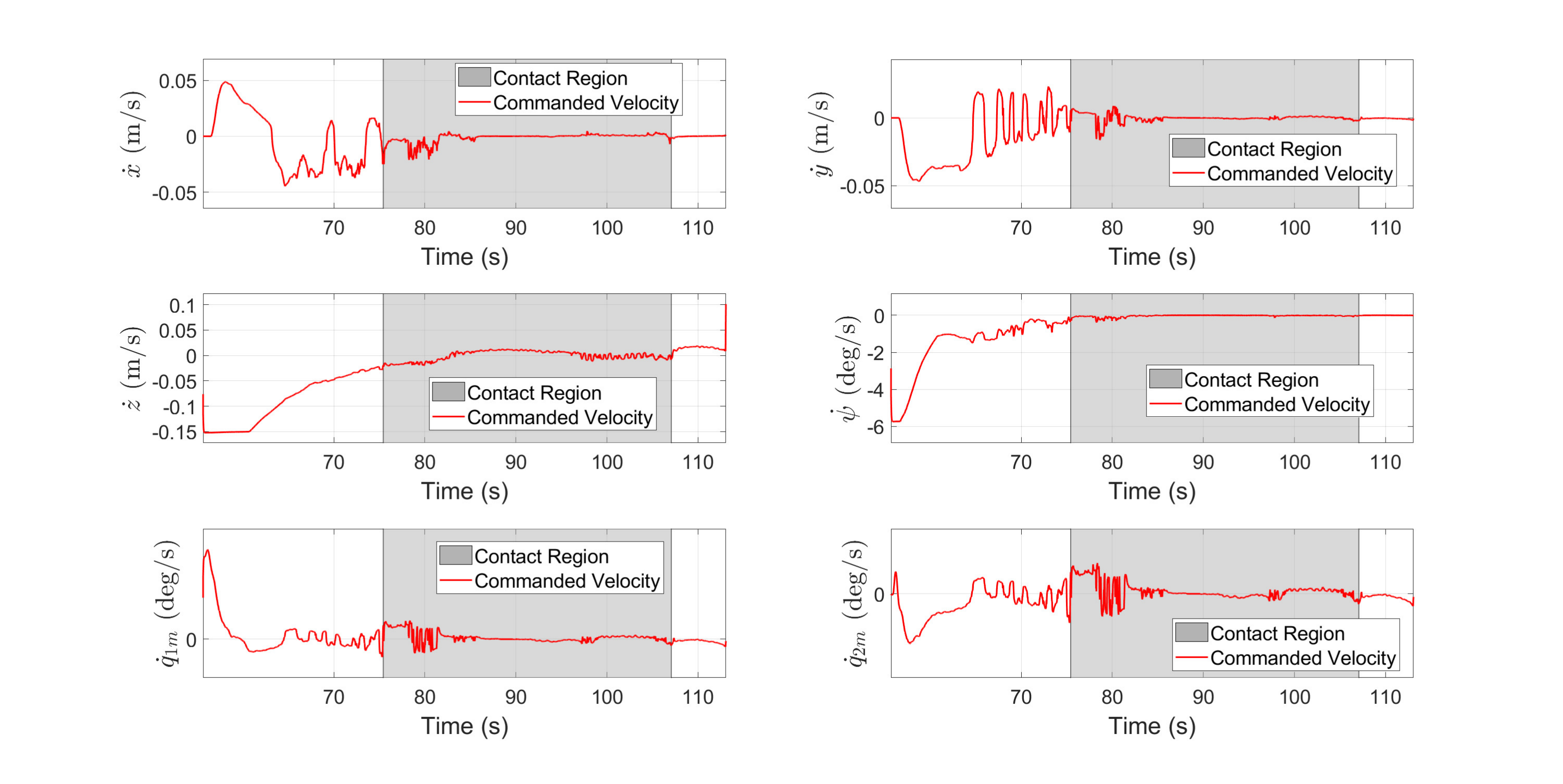}
    \caption{Controller velocity command outputs, Experiment 1. }
    \label{fig:vels_above}
\end{figure}

\subsection{Experiment 2: \textit{Below the $B=0$ curve}}
For the second experiment, the controller is enabled after the aerial manipulator acquires an initial pose very close to the surface and significantly misaligned with it. This corresponds to an initial configuration 'below' the proposed distance-alignment barrier curve, inside the undesirable region. The surface orientation and target force are identical to that of Experiment 1.

\begin{figure}[htbp]
    \centering
    \includegraphics[width=0.9\columnwidth,keepaspectratio]{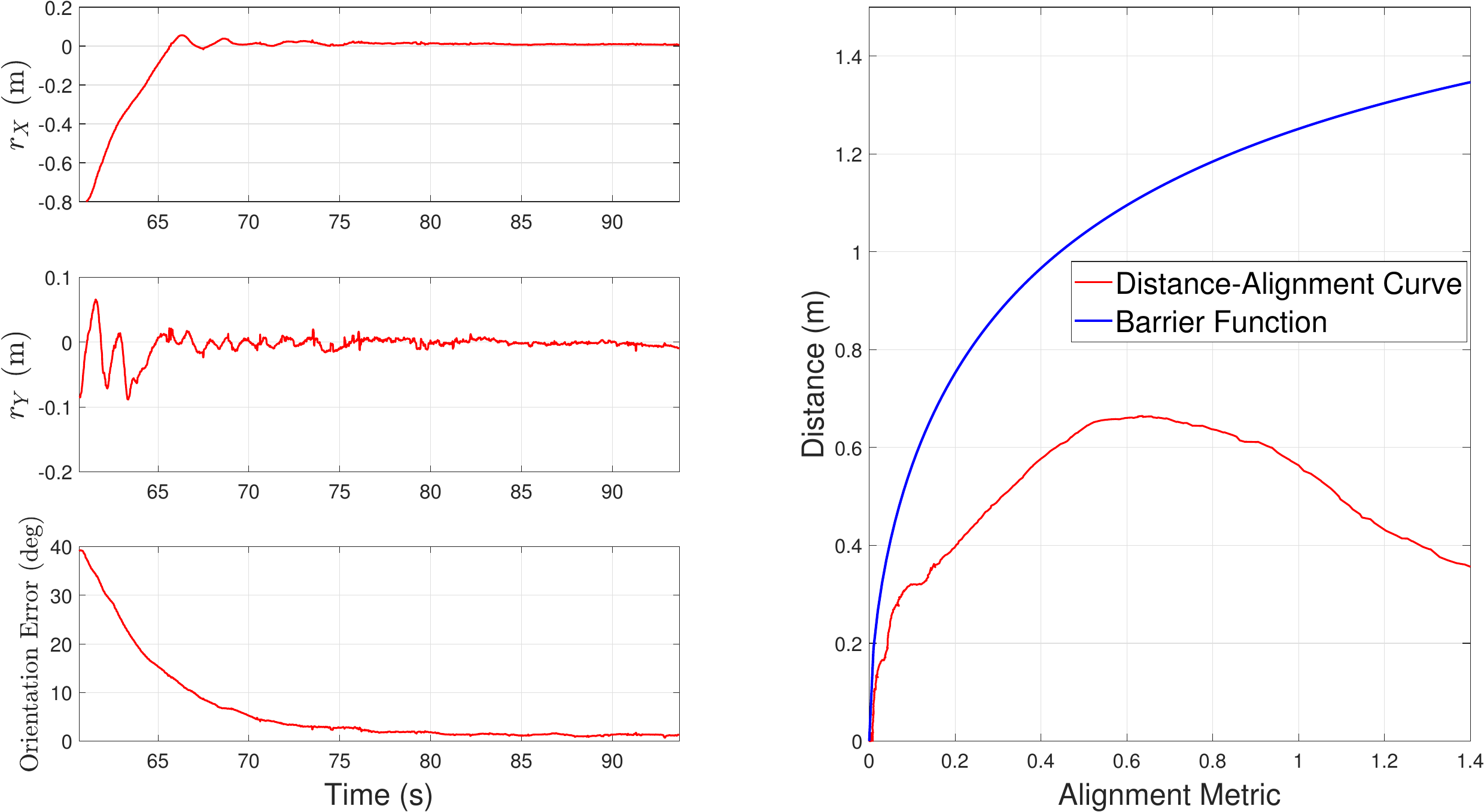}
    \caption{(Left) End-effector alignment errors (Right) Distance-alignment curve and the barrier, (Experiment 2).}
    \label{fig:cbf_below}
\end{figure}

\begin{figure}[htbp]
    \centering
    \includegraphics[width=0.7\columnwidth,keepaspectratio]{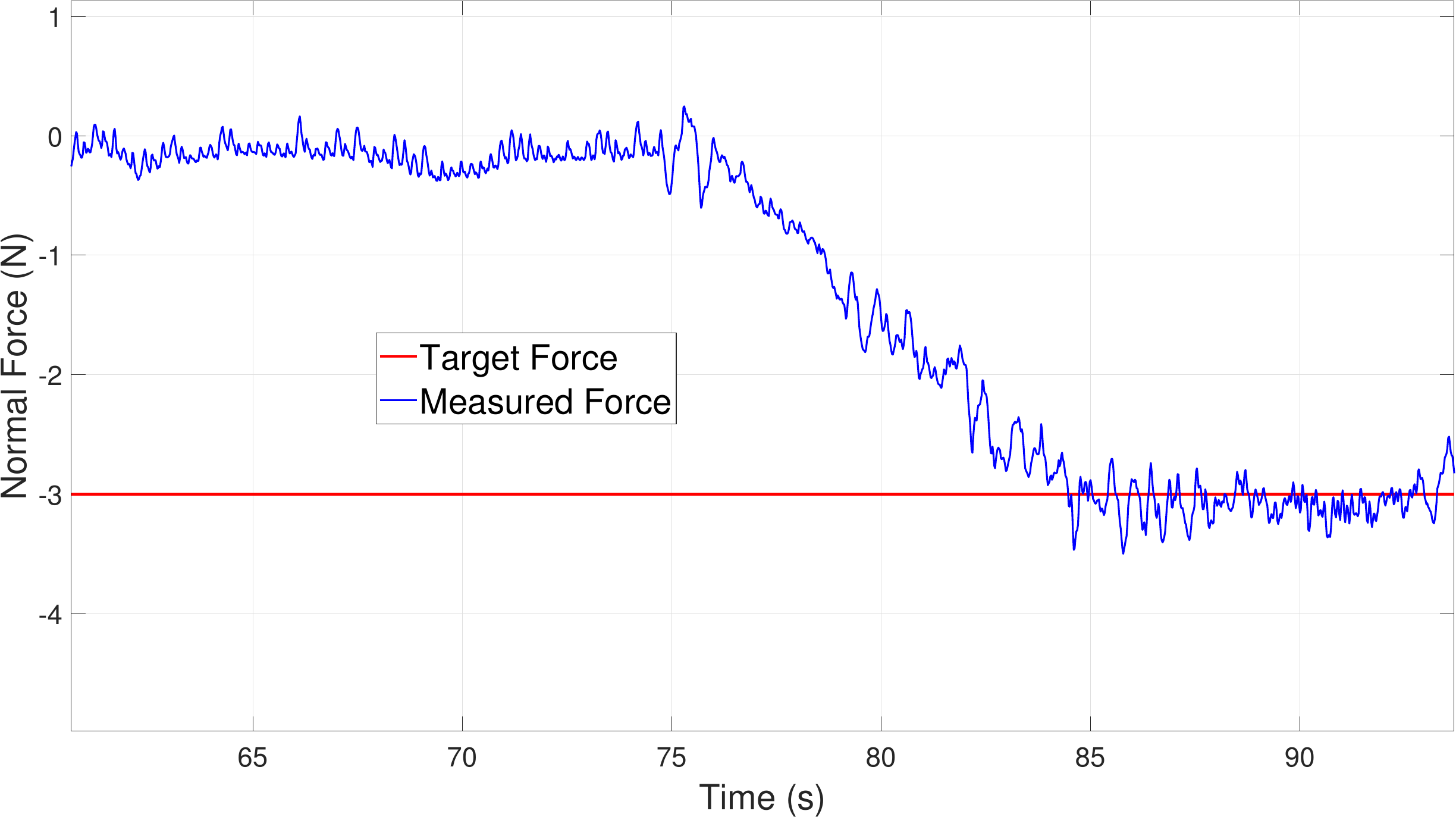}
    \caption{Measured normal force, (Experiment 2).}
    \label{fig:force_below}
\end{figure}

Figure~\ref{fig:cbf_below} shows how the controller opts to drive away from the surface initially, while improving the alignment. This is a direct outcome of trying to stay near the desired distance-alignment curve. From Figure~\ref{fig:force_below}, it can be seen that maintaining force close to the desired is unaffected by the initial conditions, with the task succeeding again.

\subsection{Experiment 3: Inclined surface}
For the third experiment, the focus is on exhibiting how the use of a dexterous robotic arm, allows the exertion of forces on surfaces of arbitrary orientation. Thus, the target surface is now at a $30^o$ inclination with respect to the world frame. The setup can be seen in the left of Figure~\ref{fig:experiment_pics}.

Figure~\ref{fig:cbf_inclined} shows that the surface orientation does not affect the controller approach, as the end-effector is kept close to the desired barrier curve still. More importantly, Figure~\ref{fig:force_inclined} demonstrates how the proposed controller is just as capable of achieving safe force exertion on a non-vertical inclined surface. 
\begin{figure}[htbp]
    \centering
    \includegraphics[width=0.9\columnwidth,keepaspectratio]{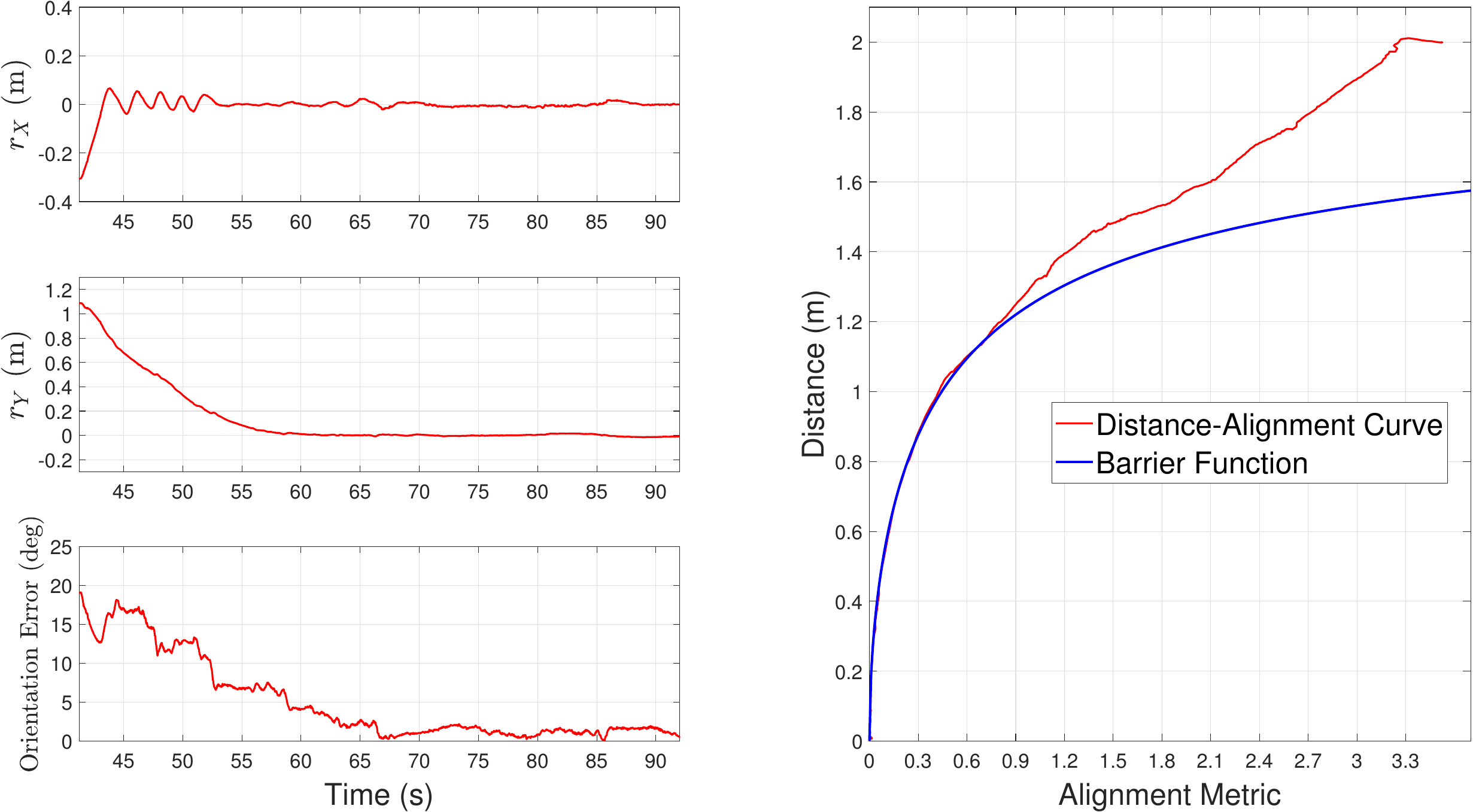}
    \caption{(Left) End-effector alignment errors (Right) Distance-alignment curve and the barrier, (Experiment 3).}
    \label{fig:cbf_inclined}
\end{figure}

\begin{figure}[htbp]
    \centering
    \includegraphics[width=0.7\columnwidth,keepaspectratio]{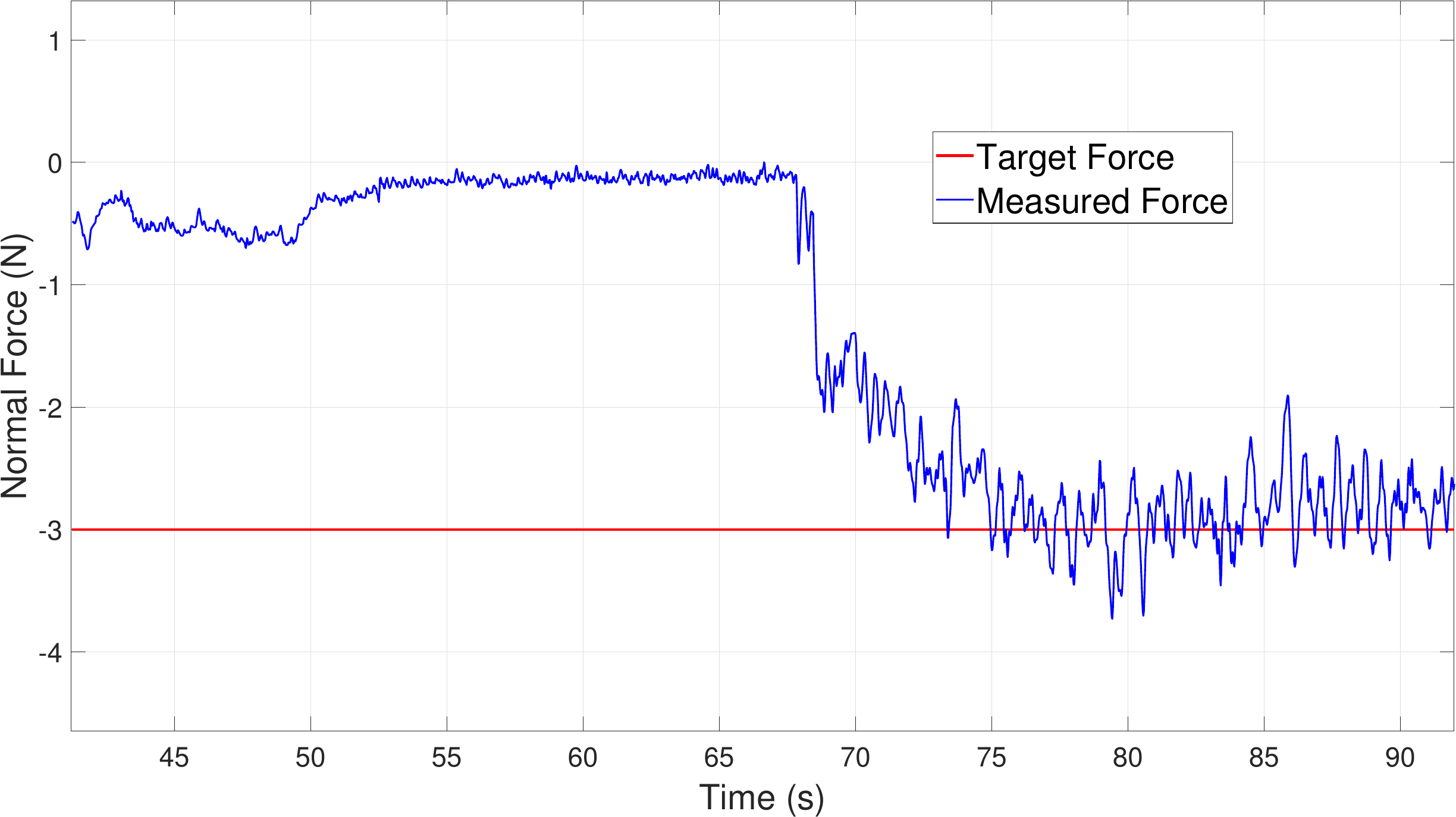}
    \caption{Measured normal force against inclined wall, (Experiment 3).}
    \label{fig:force_inclined}
\end{figure}

\subsection{Experiment 4: Different exerted forces}
The fourth experiment is dedicated to comparing the performance of the controller for various different normal force targets. The targets span from $F_d = -1N$ until $F_d = -5N$, on a vertically oriented surface. From Figure~\ref{fig:force_comparisons} it can be seen that the controller can achieve convergence to different force magnitudes.
\begin{figure}[htbp]
    \centering
    \includegraphics[width=0.7\columnwidth,keepaspectratio]{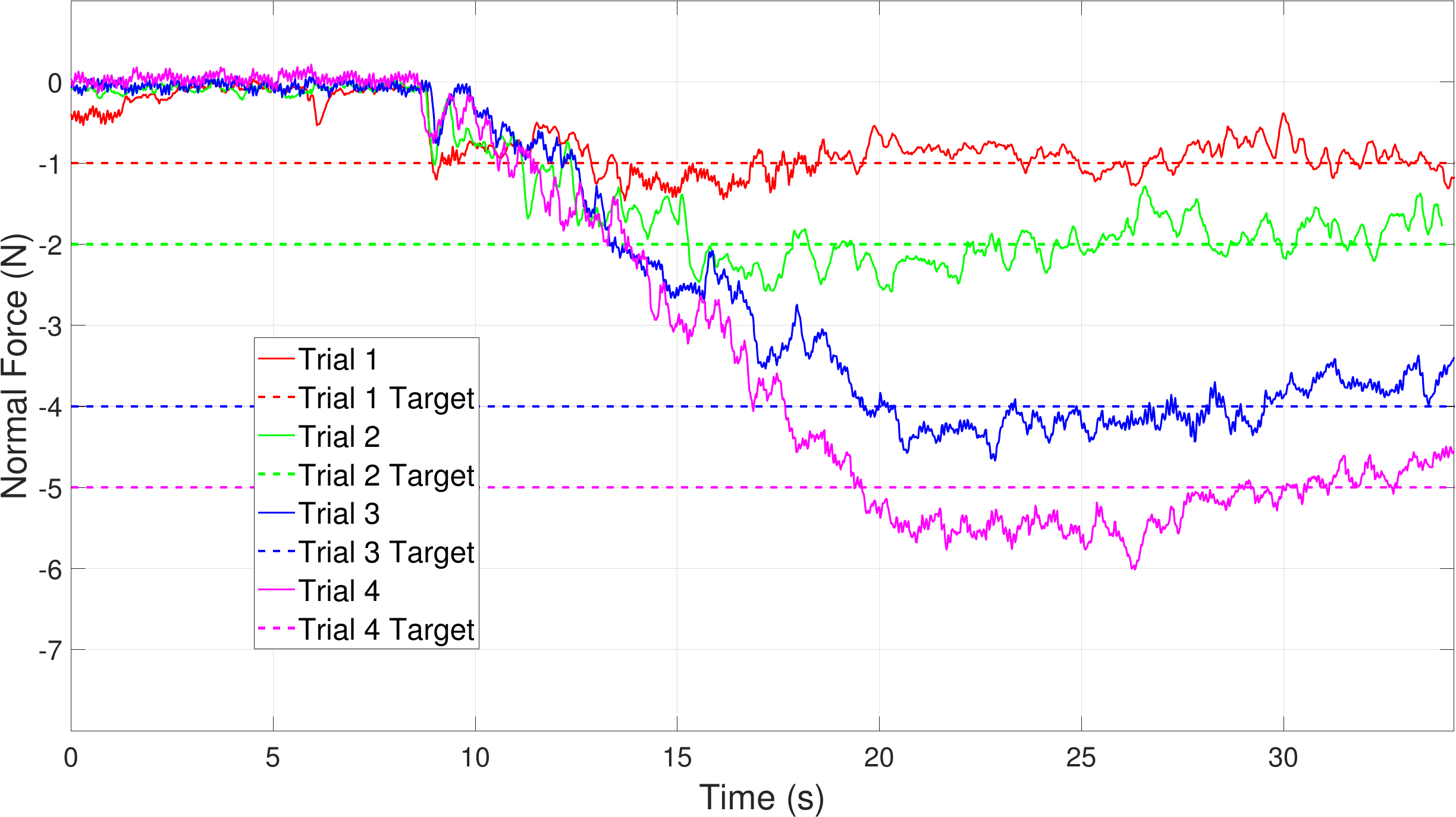}
    \caption{UAM-force exertion for different desired forces, (Experiment 4).}
    \label{fig:force_comparisons}
\end{figure}

\subsection{Experiment 5: Hard contact (large stiffness)}
In all previous experiments, the end-effector was coated with a $2cm$ layer of soft material with Young's modulus $50$kPa, and the material of the surface was a thick block of polystyrene with Young's modulus $3$MPa. In order to verify that soft contact is not necessary for the controller to achieve the task, the vertical surface was replaced by a metal sheet (Young's modulus $60$GPa), while a new end-effector is introduced, featuring metal balls for hard reduced-friction contact. The setup can be seen in the top of Figure~\ref{fig:experiment_pics}. The reduced-friction end-effector is seen in Figure~\ref{fig:metal_end_effector}.
\begin{figure}[htbp]
    \centering
    \includegraphics[width=0.5\columnwidth,keepaspectratio]{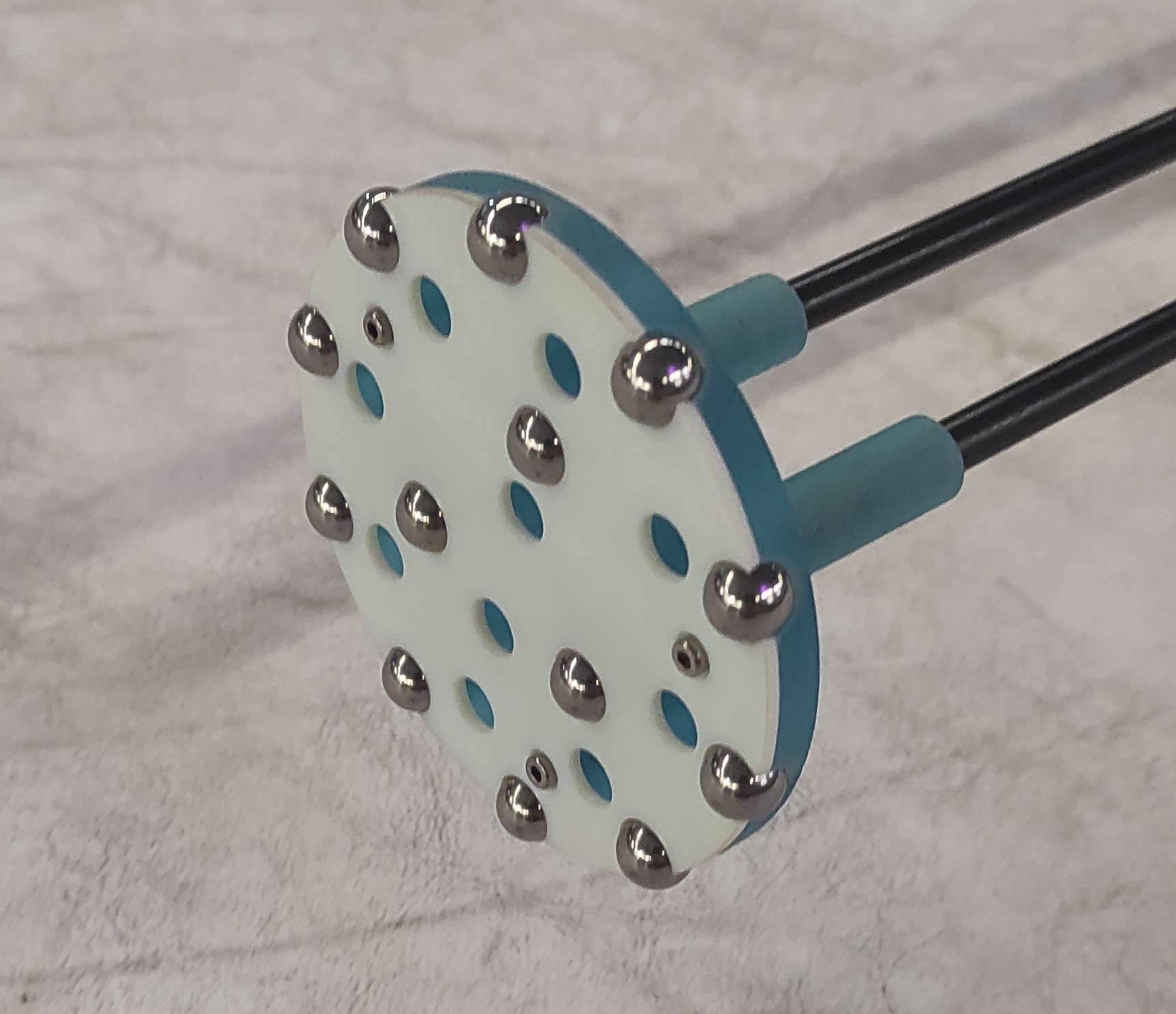}
    \caption{End-effector with metal balls for reduced-friction hard contact.}
    \label{fig:metal_end_effector}
\end{figure}

Figures~\ref{fig:cbf_hard},\ref{fig:forces_hard} help to conclude that the proposed controller is capable achieving stable force control under large stiffness, similar to the soft contact cases.

\begin{figure}[htbp]
    \centering
    \includegraphics[width=0.9\columnwidth,keepaspectratio]{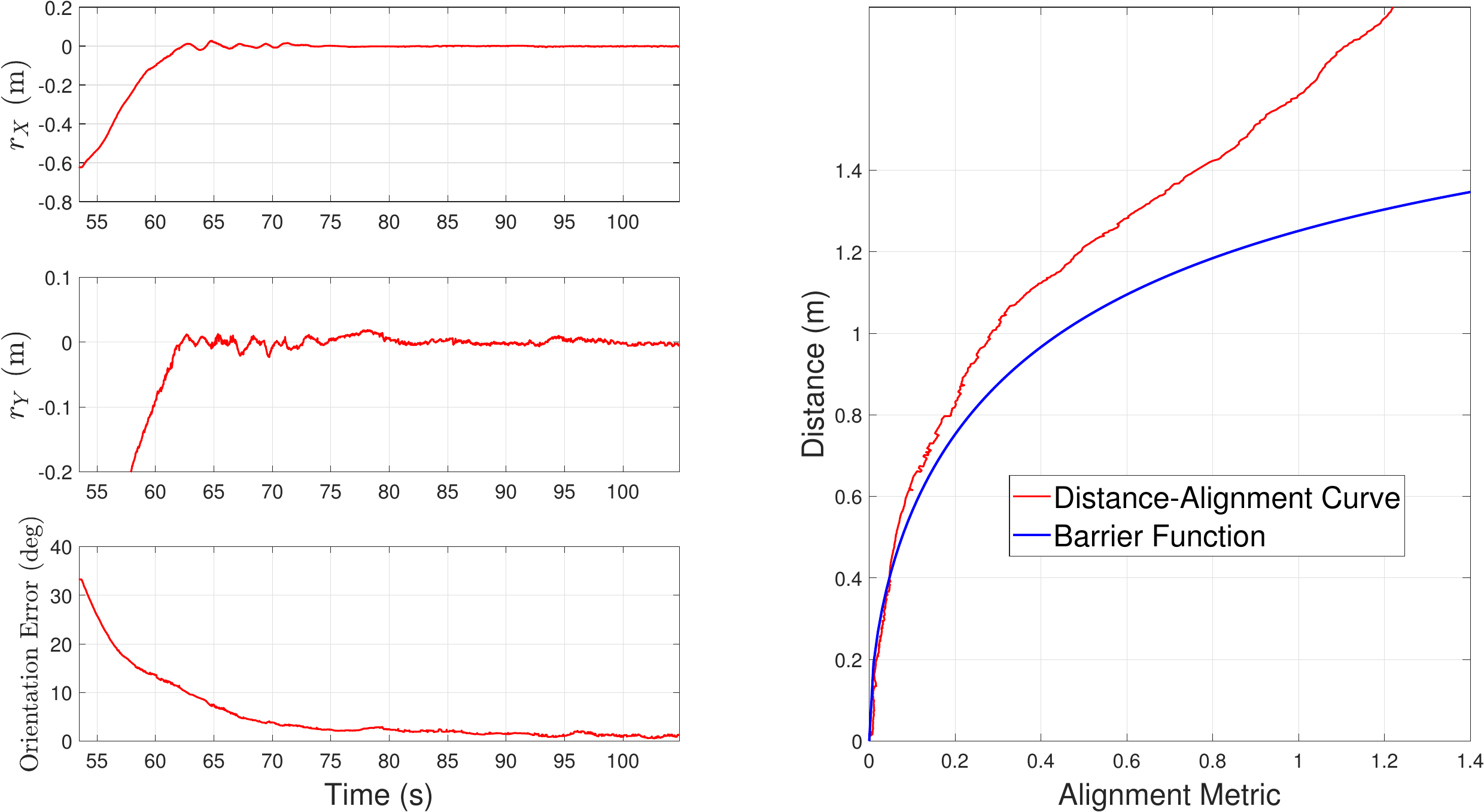}
    \caption{(Left) End-effector alignment errors (Right) Distance-alignment curve and the barrier (Experiment 5)}
    \label{fig:cbf_hard}
\end{figure}
\begin{figure}[htbp]
    \centering
    \includegraphics[width=0.7\columnwidth,keepaspectratio]{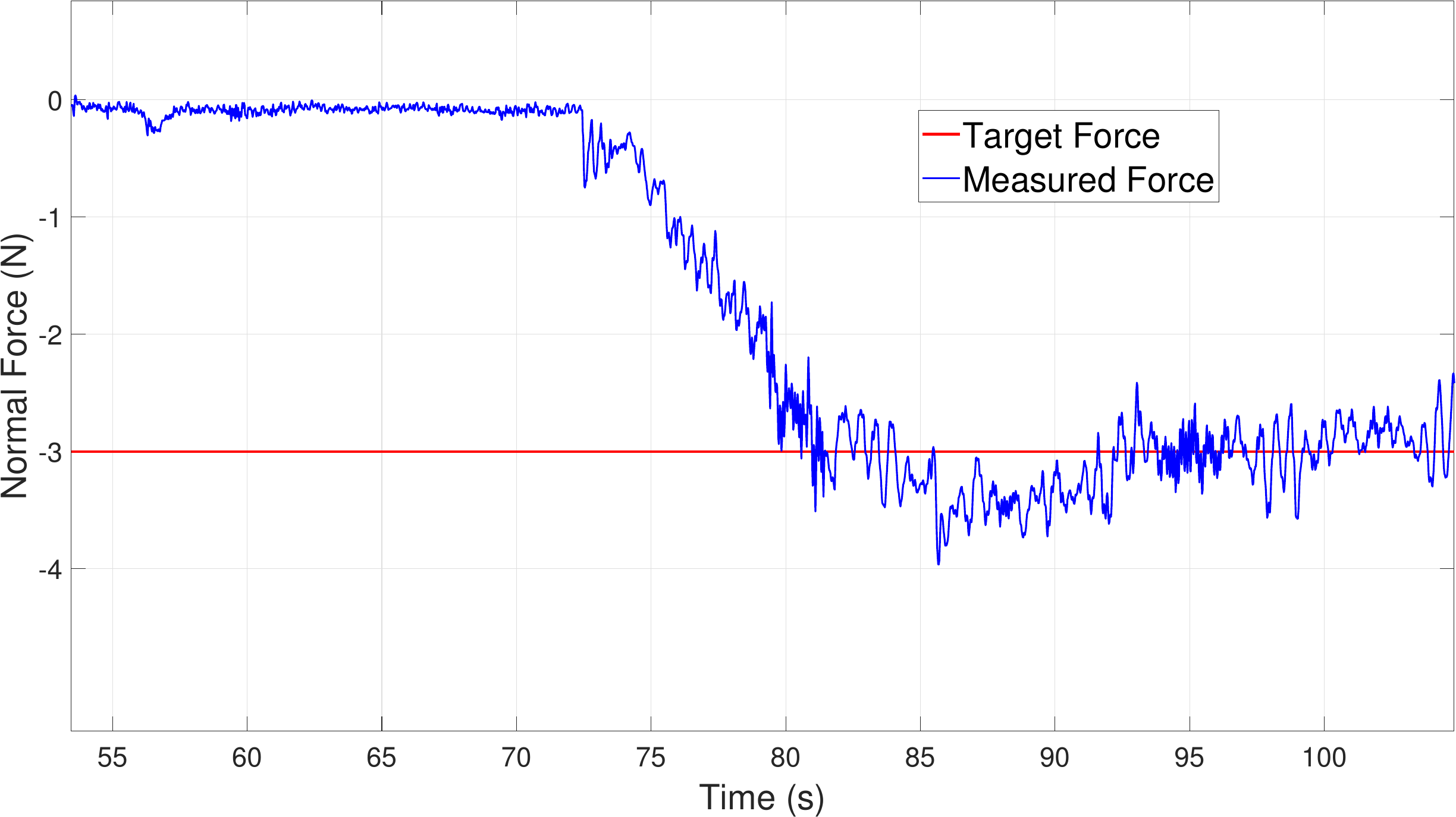}
    \caption{Measured normal force in the case of hard contact, (Experiment 5).
    }
    \label{fig:forces_hard}
\end{figure}

\subsection{Experiment 6: Contact/Force application while moving}
There are real-world applications requiring  contact with the environment and applying force while moving or tracking a trajectory. For the last experiment, the desired contact position along the planar axes of the surface is modified with application of the same controller to achieve motion while tracking the desired normal force.

The results from this experiment are seen in Figures~\ref{fig:force_motion_cbf},\ref{fig:force_motion_ft}. From Figure~\ref{fig:force_motion_cbf} the end-effector successfully performs two distinct motion profiles on the surface plane, while maintaining the desired relationship between distance and alignment.
\begin{figure}[htbp]
    \centering
    \includegraphics[width=\columnwidth,keepaspectratio]{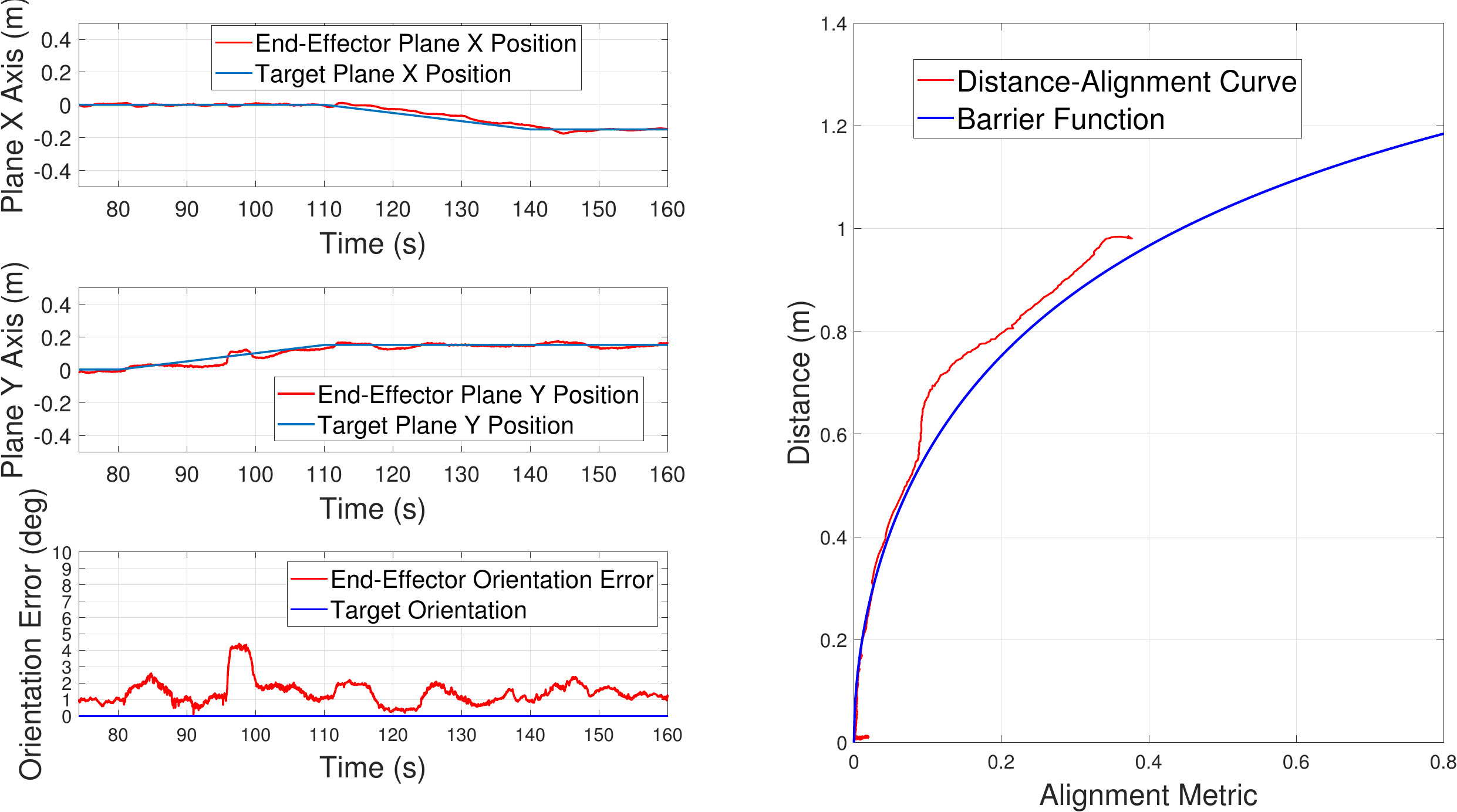}
    \caption{Motions on surface plane and corresponding barrier curve, (Experiment 6)
}
    \label{fig:force_motion_cbf}
\end{figure}

All the measured forces and torques are shown in Figure~\ref{fig:force_motion_ft}, where the UAM maintains normal force exertion throughout the motion, while indeed all other forces and torques remain close to zero. In all other experiments with static force exertion, the corresponding planar forces and torques are likewise close to zero.
\begin{figure}[htbp]
    \centering
    \includegraphics[width=\columnwidth,keepaspectratio]{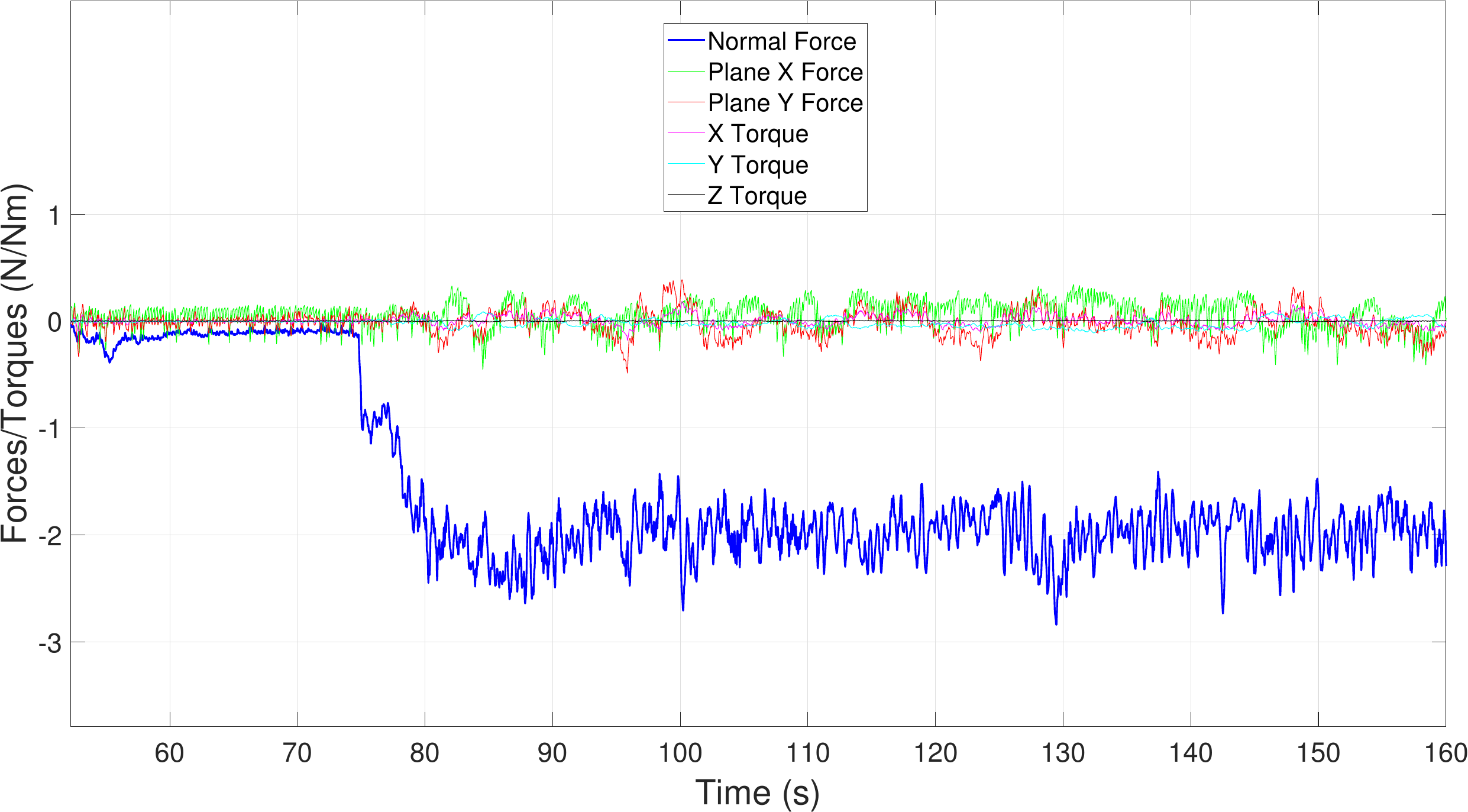}
    \caption{All measured forces and torques during motion with normal force exertion, (Experiment 6)
}
    \label{fig:force_motion_ft}
\end{figure}

\section{Conclusion}
\label{sec:concl}
In this paper, an optimization-based control method was presented for the problem of safe normal force exertion on surfaces of arbitrary orientation. The controller utilizes CBFs for prescribing a relationship between the distance to the surface and the alignment of the end-effector with it, thus ensuring contacts occur under strict, safe conditions. An in-depth analysis of the control method is provided, including feasibility, continuity and stability proofs. Extensive experiments were carried out, using different surface materials, initial conditions and force targets, establishing the capacity of the proposed control method to achieve aerial manipulator force control.
\section*{Acknowledgements}
This work was partially supported by the NYUAD Center for Artificial Intelligence and Robotics (CAIR), funded by Tamkeen under the NYUAD Research Institute Award CG010. Experimental studies were partially performed in NYUAD's Kinesis Lab, Core Technology Platform, Abu Dhabi, UAE. The authors thank Mr. Nikolaos Giakoumidis for his valuable technical support.

\section*{Appendix A - Proofs of Propositions}
\textbf{Proof of Proposition \ref{prop:nonzerograd}}: Using Lemma \ref{lemma:pdresulta} it can be seen that  $\frac{\partial}{\partial z}B(q) = \frac{\partial}{\partial z} Z(q)  = 1~\forall q$. $\qedsymbol$

\textbf{Proof of Proposition \ref{prop:pi}}: The CBF constraint in \eqref{q_m_cbf} guarantees strict bounds $\underline{q}_m,\overline{q}_m$ for $q_m$, while the first constraint of \eqref{eq:uH} guarantees $\Dot{B}(q)$ is non-decreasing as $B(q)\rightarrow0$, thus ensuring positive invariance of $\setP$. $\qedsymbol$

\textbf{Proof of Proposition \ref{prop:feasF}}: Note that, by the definition of $\underline{b}(q)$, $\overline{b}(q)$ in \eqref{eq:bdef}, 
it is always true that   $\underline{b}(q) \leq \overline{b}(q)$. Thus, the constraint $\underline{b}(q) \leq \mu \leq \overline{b}(q)$ alone is always feasible for any $q$. It remains to study when this conflicts with $\nabla_q B(q)^{\top} \mu \geq -\kappa_B(B(q))$.
Consider the linear program (LP):
\begin{eqnarray}
\max_\mu  \hspace{-5pt}  &&  \nabla_q B(q)^{\top} \mu ~ \mbox{such that} \nonumber \\
&& \underline{b}(q) \leq \mu \leq \overline{b} (q). \nonumber
\end{eqnarray}
Then $\underline{b}(q) \leq \mu \leq \overline{b}(q)$ is compatible with the CBF inequality $\nabla_q B(q)^{\top} \mu \geq -\kappa_B(B(q))$ if and only if the maximum value of this LP, which is the left-hand side of the CBF inequality and represents the highest $\dot{B}$ achievable inside the box $\underline{b} \leq \mu \leq \overline{b}$, is greater than or equal to the right-hand side of the CBF inequality. This LP is trivial to solve since it is decoupled for the variables $\mu_i$. The optimal $\mu_i$ can be obtained by solving an one-dimensional LP for each variable:
\begin{eqnarray}
&& \max_{\mu_i} \ \ \  \frac{\partial B}{\partial q_i}(q) \mu_i ~ \mbox{such that} \nonumber \\
&& \underline{b}_i(q) \leq \mu \leq \overline{b}_i (q). \nonumber
\end{eqnarray}
\noindent The optimal solution can be obtained by inspection, and is: $\mu_i = \underline{b}_i(q)$ if $\frac{\partial B}{\partial q_i}(q) < 0$, $\mu_i = \overline{b}_i(q)$ if $\frac{\partial B}{\partial q_i}(q) > 0$ and any $\mu_i$ if $\frac{\partial B}{\partial q_i}(q) = 0$. That is, $\mu = b^*(q)$, with $b^*$ according to Definition \ref{def:hlvlu}. With this choice, the optimal value of the LP is $\nabla_q B(q)^\top b^*(q)$, and thus it is possible to obtain the necessary and sufficient condition $\nabla_q B(q)^\top b^*(q) \geq -\kappa_B(B(q))$, i.e., $q \in \setF$. $\qedsymbol$

\textbf{Proof of Proposition \ref{prop:zero}}: Application of $\mu=0$ in \eqref{eq:uH}, results in
$
    -\kappa_B \big( B(q) \big) < 0,~\mbox{and }  \underline{b}(q) < 0 < \overline{b}(q). 
$
Both of the above inequalities are already guaranteed from the positive invariance of $\setP$, the properties of $\kappa_B$ (Definition \ref{def:hlvlu}) and $\underline{b}, \overline{b}$ (Definition \ref{def:lim}). Consequently, $\mu=0$ is a strictly feasible solution
. $\qedsymbol$

\textbf{Proof of Proposition \ref{prop:unique}}: From Proposition \ref{prop:feasF}, $q \in \setF$ implies that a solution exists. On the other hand, from convex function theory, it is known that a strictly convex function minimized over a convex set, has a unique optimal solution,  provided that a solution exists in the first place. So it remains to show that the optimization problem in \eqref{eq:uH} is a strictly convex optimization problem.

Indeed, in \eqref{eq:uH}, the constraint set is comprised of linear inequalities and is thus convex. The quadratic optimization problem objective function of \eqref{eq:uH} can be written in the form
\begin{equation*}
    \min_{\mu} ~ ~ \frac{1}{2}\mu^T H \mu + f^T\mu
\end{equation*}
with $H=2\nabla_q r_Z \nabla_q r_Z^{\top} +  2E$. Since $H$ is positive definite (Lemma \ref{lemma:pod}), the objective function is strictly convex. $\qedsymbol$

\textbf{Proof of Proposition \ref{eq:propcontinuity}}: The constraint functions are continuous in $q$. The objective function matrices $H, f$ are also continuous in $q$ and $H$ is positive definite (see Lemma \ref{lemma:pod}). From \cite{Lipshitz}, it can be shown that the above conditions, along with $\mu=0$ being a strictly feasible solution (from Proposition 2), are sufficient to prove continuity of the optimal solution $u_H(q)$ in $q$. $\qedsymbol$


\textbf{Proof of Proposition \ref{prop:lyapst}}: By definition, the solution $\mu=u_H(q)$ is the unique (see Proposition \ref{prop:unique}) minimizer of $W$, or $W(q,u_H(q))\leq W(q,\mu)$ for any feasible solution $\mu$. Since $\mu=0$ is feasible (Proposition \ref{prop:zero}), the particular inequality $W(q,u_H(q))\leq W(q,0)$ holds. Using this inequality, the fact that $\dot{r}_Z= \nabla_q r_Z(q)\dot{q}$, expanding the squared norm and simplifying the resulting expression, it is possible to obtain
\begin{equation*}
    2\frac{\partial V_F}{\partial r_Z}\dot{r}_Z  + \|\dot{r}_Z\|^2 + \dot{q}^{\top} E \dot{q} \leq 0.
\end{equation*}
Using the fact that $\frac{\partial V_F}{\partial r_Z} \dot{r}_Z  = \dot{V}_F$, it is deduced that $\dot{V}_F \leq -\frac{1}{2} (\|\dot{r}_Z\|^2 + \dot{q}^{\top} E \dot{q})$ and the desired result easily follows from Lemma \ref{lemma:pod}. $\qedsymbol$  

\textbf{Proof of Proposition \ref{prop:Adot}}: The Karush-Kuhn Tucker conditions (KKT) for the problem in \eqref{eq:uH} will be necessary. They are:
\begin{eqnarray}
\label{eq:kkt}
    && \mbox{Stationarity}:  \nonumber \\
    && \nabla_q r_Z (\nabla_q r_Z^{\top}\mu{+}\kappa_F) {-} \nabla_q B \lambda - \underline{\lambda} + \overline{\lambda} + E \dot{q} = 0. \nonumber \\
    && \mbox{Complementarity}:  \nonumber \\
    &&  (\nabla_q B^{\top}\mu+\kappa_B(B))\lambda = 0 \nonumber \\
    && (\mu_i-\underline{b}_i) \underline{\lambda}_i = 0 \ , \  (\overline{b}_i-\mu_i) \overline{\lambda}_i = 0\nonumber \\
    && \mbox{Primal feasibility}:  \nonumber \\
    && \nabla_q B^{\top} \mu \geq -\kappa_B (B) \ , \ \underline{b} \leq \mu \leq \overline{b}  \nonumber \\
    && \mbox{Dual feasibility}:  \nonumber \\
    && \lambda \geq 0 \ , \ \underline{\lambda} \geq 0 \ , \ \overline{\lambda} \geq 0 
\end{eqnarray}
\noindent in which $\mu \in \mathbb{R}^n$ is the primal variable and $\lambda \in \mathbb{R}, \underline{\lambda}, \overline{\lambda} \in \mathbb{R}^n$ are the dual variables.

From the ``Stationarity'' condition in  \eqref{eq:kkt} for the component $q_3=z$, using the facts that (i)  $\epsilon_3=0$ (see Definition \ref{def:W}), (ii) $\frac{\partial r_Z}{\partial z} = \frac{\partial Z}{\partial z} = 1$ (see Lemma \ref{lemma:pdresulta}) and (iii) the assumption that no limits on $\dot{z}$ are imposed (and thus $\underline{\lambda}_3 = \overline{\lambda}_3=0$),  
it can be obtained 
\begin{equation}
\label{eq:stationarity1}
 \nabla_q r_Z^{\top}\dot{q} + \kappa_Z = \lambda.   
\end{equation}
    
    Consequently, using this information for the whole   ``Stationarity'' condition in  \eqref{eq:kkt}:
    \begin{equation}
    \label{eq:stationarity2}
        (\nabla_q r_Z-\nabla_q B) (\nabla_q r_Z^{\top} \dot{q}+\kappa_F)  = -E\dot{q} +\underline{\lambda}-\overline{\lambda}.
    \end{equation}    
    It is now necessary to consider several facts, that can be verified utilizing the standard rules of calculus and the definitions of $r_Z$ and $B$: (i) $\nabla_q r_Z = \nabla_q Z$, (ii) $ \nabla_q r_Z {-} \nabla_q B=  \nabla_q Z {-} (\nabla_q Z - \nabla_q \kappa_A(A)) = \nabla_q \kappa_A(A) =  \frac{\partial \kappa_A}{\partial A}\nabla_q A$, (iii) $\dot{q}^{\top}\nabla_q r_Z  = \dot{Z}$ and (iv) $\dot{q}^{\top}\nabla_q A = \dot{A}$. Thus, from \eqref{eq:stationarity2}, applying these facts after pre-multiplying both sides of the equation by $\dot{q}^T$, the following conclusion arises:
    \begin{equation}
    \label{eq:stationarity3}
        \frac{\partial \kappa_A}{\partial A}\dot{A} (\dot{Z}+\kappa_F) = -\dot{q}^{\top}E\dot{q} +\dot{q}^{\top}\underline{\lambda}-\dot{q}^{\top}\overline{\lambda}.
    \end{equation}  
    From the ``Complementarity'' condition in  \eqref{eq:kkt}, $\underline{\lambda}^{\top}\dot{q}=\underline{\lambda}^{\top}\underline{b}$, $\overline{\lambda}^{\top}\dot{q}=\overline{\lambda}^{\top}\overline{b}$. Furthermore, ``Dual Feasibility'' implies that $\underline{\lambda}, \overline{\lambda} \geq 0$, and since $q \in \setP$, $\underline{b} < 0$ and $\overline{b} > 0$. Thus, $\dot{q}^{\top}\underline{\lambda}-\dot{q}^{\top}\overline{\lambda} = \underline{\lambda}^{\top}\underline{b} - \overline{\lambda}^{\top}\overline{b} \leq 0$. Consequently, from \eqref{eq:stationarity3}
    \begin{equation}
    \label{eq:impconclusionlemmaA}
        \frac{\partial \kappa_A}{\partial A}\dot{A} (\dot{Z}+\kappa_F) \leq -\dot{q}^{\top}E\dot{q}.
    \end{equation}
    Since $\kappa_A$ is $\kappa$-like, $\frac{\partial \kappa_A}{\partial A} > 0$. Furthermore, from \eqref{eq:stationarity1} and ``Dual Feasibility'', $\dot{Z}+\kappa_F = \lambda \geq 0$. To complete the proof, two cases are considered: either $\dot{Z}+\kappa_F  > 0$ or $\dot{Z}+\kappa_F = 0$. For the former case, it is easy to see that \eqref{eq:impconclusionlemmaA} implies the desired result. For the latter, note that from \eqref{eq:impconclusionlemmaA}, $0 \leq -\dot{q}^{\top}E\dot{q} $. But since $\dot{q}^{\top}E\dot{q} \geq 0$, this forces $\dot{q}^{\top}E\dot{q}  = 0$, that is, $E\dot{q} = 0$. This implies that all the components of $\dot{q}$, except maybe $\dot{z}$, vanish. But in this case, $\dot{A} = 0$, since $A$ does not depend on $z$ because it depends on $r_X=X$, $r_Y=Y$ and $r_O$, and none of these depend on $z$ (see Lemma \ref{lemma:pdresulta}). $\qedsymbol$

\textbf{Proof of Proposition \ref{prop:zeroset}}: From \eqref{eq:skkt} in Lemma \ref{lemma:skkt}, two mutually exclusive cases are considered: either $\kappa_F = \kappa_F(Z,F-F_d) = 0$ or $\kappa_F = \kappa_F (Z,F-F_d) \not = 0$.

\vspace{5pt}

\underline{Assuming that $\kappa_F=0$}: in this case, it will be shown that $q \in \setSi$. The variables can be selected as $\underline{\lambda}_{m,j} = \overline{\lambda}_{m,j} = 0$.

$\bullet$ \emph{Condition (i)-(a):} from Proposition \ref{def:potential}, $\kappa_F = 0$ is equivalent to $Z=Z_d$ (or $F=F_d$). 

$\bullet$ \emph{Condition (i)-(b):} if $\kappa_F = 0$, from Definition \ref{def:potential} it can be concluded that $Z-Z_d = 0$, that is, $Z=Z_d$. From the ``Primal Feasibility'', $B = Z-Z_d^* - \kappa_{A}(A) \geq 0$, and thus $A \leq \kappa_{A}^{-1}(Z_d-Z_d^*)$, in which the fact that, from Definition \ref{def:A}, $\kappa_{A}$ is strictly increasing, and hence invertible, was used. 

$\bullet$ \emph{Condition (i)-(c):} comes  from ``Primal Feasibility''.
\vspace{5pt}
\underline{Assuming that $\kappa_F \not= 0$}: in this case, it will be shown that $q \in \setSii$. Obtaining the remaining variables $\underline{\lambda}_{m,j} = \overline{\lambda}_{m,j}$ is more intricate, and will be  explained when discussing the condition (ii)-(e).
 
 $\bullet$ \emph{Condition (ii)-(a):} the ``Dual Feasibility'' condition for $\lambda = \kappa_F(Z,F-F_D)$ implies that $\kappa_F(Z,F-F_D) \geq 0$. Due to the properties of the function $\kappa_F$ (see Definition \ref{def:potential}), this is true only if $Z \geq Z_d$. Since $\kappa_F \not =0$, then it remains that $Z > Z_d$ (and also $\kappa_F > 0$).

$\bullet$ \emph{Condition (ii)-(b):} since $\kappa_F \not= 0$, from the ``Stationarity'' condition in \eqref{eq:skkt} for the variable $x$, $\frac{\partial }{\partial x} \kappa_{A}(A) = \frac{\partial }{\partial y} \kappa_{A}(A) = \frac{\partial }{\partial \psi} \kappa_{A}(A) = 0$. 
Using the chain rule:

\begin{eqnarray}
    && \frac{\partial }{\partial x} \kappa_{A}(A) = \frac{\partial \kappa_A}{\partial A}\frac{\partial A}{\partial x} = \nonumber \\
    &&\frac{\partial \kappa_A}{\partial A}\Bigg(\frac{\partial V_A^{XY}}{\partial r_X}\underbrace{\frac{\partial r_X}{\partial x}}_{\frac{\partial X}{\partial x}=1} + \frac{\partial V_A^{XY}}{\partial r_Y}\underbrace{ \frac{\partial r_Y}{\partial x}}_{\frac{\partial Y}{\partial x}=0} + \frac{\partial \aof}{\partial r_O}\underbrace{ \frac{\partial r_O}{\partial x}}_{\ez^\top \frac{\partial \zc}{\partial x}=0}  \Bigg)  
    \nonumber \\ 
    && = \frac{\partial \kappa_A}{\partial A}\frac{\partial V_A^{XY}}{\partial r_X} \nonumber
\end{eqnarray}
\noindent in which Lemma \ref{lemma:pdresulta} was used. Thus $\frac{\partial \kappa_A}{\partial A} \frac{\partial }{\partial r_X}V_A^{XY} = 0$. Since $\kappa_A$ is strictly increasing (Definition \ref{def:A}), $\frac{\partial \kappa_A }{\partial A} > 0$, and consequently it is possible to conclude that  $\frac{\partial}{\partial r_X} V_A^{XY} = 0$. A similar analysis implies that  $\frac{\partial }{\partial r_Y}V_A^{XY} = 0$ as well.  Since $V_A^{XY}$ is a Lyapunov-like function (see Definition \ref{def:A}), this can only be true when $r_X=r_Y=0$.

$\bullet$ \emph{Condition (ii)-(c):} the next condition comes from analyzing   the ``Stationarity'' condition in \eqref{eq:skkt} for $\psi$. Note that:
\begin{eqnarray}
\label{eq:concpsi}
&&\frac{\partial }{\partial \psi} \kappa_{A}(A) = \frac{\partial \kappa_A}{\partial A}\frac{\partial A}{\partial \psi} = \nonumber \\
&&\frac{\partial \kappa_A}{\partial A}\Bigg(\underbrace{\frac{\partial V_A^{XY}}{\partial r_X}}_{0} \frac{\partial r_X}{\partial \psi}{+}\underbrace{\frac{\partial V_A^{XY}}{\partial r_Y}}_{0} \frac{\partial r_Y}{\partial \psi}{+} \frac{\partial \aof}{\partial r_O} \frac{\partial r_O}{\partial \psi}\Bigg) \nonumber \\
&& = \frac{\partial \kappa_A}{\partial A}\frac{\partial \aof}{\partial r_O} \frac{\partial r_O}{\partial \psi}
\end{eqnarray}
\noindent in which Condition (ii)-(b) was used.
Considering that $\kappa_A$ and $\aof$ are strictly increasing (see Definition \ref{def:A}) and Lemma \ref{lemma:partialz}, it can be concluded from \eqref{eq:concpsi} that  $\frac{\partial r_O}{\partial \psi} = \ez^{\top}(\az \times \zc(q))=0$. 

$\bullet$ \emph{Condition (ii)-(d):} from the ``Complementarity'' condition, it can be concluded that since $\lambda = \kappa_F \not = 0$, $\kappa_B(B)=0$. Due to the properties of the function $\kappa_B$ (Definition \ref{def:hlvlu}), this implies $B=0$. This is condition (ii)-(d).

$\bullet$ \emph{Condition (ii)-(e):} from the ``Stationarity'' condition in \eqref{eq:skkt} for $q_{m,j}$,
\begin{equation}
\label{eq:impconclusion3}
\kappa_F \frac{\partial }{\partial q_{m,j}} \kappa_{A}(A) =  \underline{\lambda}_{m,j}{-}\overline{\lambda}_{m,j}. 
\end{equation}
$\frac{\partial }{\partial q_{m,j}} \kappa_{A}(A)$ can be written, using the chain rule, as
 \[
 \frac{\partial \kappa_A}{\partial A} \Bigg( \underbrace{\frac{\partial V_A^{XY}}{\partial r_X}}_{0}\frac{\partial r_X}{\partial q_{m,i}}{+}\underbrace{\frac{\partial V_A^{XY}}{\partial r_Y}}_{0}\frac{\partial r_Y}{\partial q_{m,i}}{+}\frac{\partial \aof}{\partial r_O}\frac{\partial r_O}{\partial q_{m,i}} \Bigg) 
 \]
\noindent in which $r_X=r_Y=0$ (see Condition (ii)-(c)) and the fact that $V_A^{XY}$ is Lyapunov-like was used. Thus,  since $\kappa_A$ is $\kappa$-like, $\frac{\partial \kappa_A}{\partial A} > 0$, and, together with Lemma \ref{lemma:partialz} and the fact that $\kappa_F > 0$ (see Condition (ii)-(a)), \eqref{eq:impconclusion3} can be written as:
\begin{equation}
\label{eq:lambdacond}
    (\ez \times \widecheck{n}_j)^{\top}\zc = \frac{\underline{\lambda}_{m,j}{-}\overline{\lambda}_{m,j}}{ (\partial \kappa_A/\partial A)\kappa_F}.
\end{equation}
If $q_{m,j}$ is saturated below ($q_{m,j} = \underline{q}_{m,j}$), from the ``Dual Feasibility'' and ``Complementarity Condition'' $\overline{\lambda}_{m,j}=0$ and $\underline{\lambda}_{m,j} \geq 0$. Thus, $(\ez \times \widecheck{n}_j)^{\top}\zc$ is nonnegative. Similarly, when it is saturated above  ($q_{m,j} = \overline{q}_{m,j}$), $\underline{\lambda}_{m,j}=0$ and $\overline{\lambda}_{m,j} \geq 0$ and $(\ez \times \widecheck{n}_j)^{\top}\zc$ is nonpositive. Finally, when it is non-saturated, $\overline{\lambda}_{m,j}=\underline{\lambda}_{m,j}=0$, and thus $(\ez \times \widecheck{n}_j)^{\top}\zc=0$. With these conditions, $\underline{\lambda}_{m,j}$ and $\overline{\lambda}_{m,j}$ can be obtained from \eqref{eq:lambdacond}.

$\bullet$ \emph{Condition (ii)-(f):} comes  from ``Primal Feasibility''. $\qedsymbol$
%
%
\section*{Appendix B - Proofs of Lemmas}
 \textbf{Proof of Lemma \ref{lemma:VFlyap}}: The proof uses  the well-known fact that if $g: \mathbb{R} \mapsto \mathbb{R}$ is $\kappa$-like, then $G(r_Z) \triangleq \int_0^{r_Z}g(\xi)d\xi$ is Lyapunov-like.  Due to the properties of $F$ (see Assumption \ref{assump:force}) and $\kappa_F$ (see Definition \ref{def:potential}), and the fact that the composition of $\kappa$-like functions is also $\kappa$-like, it is possible to conclude that $g(\xi) = \kappa_F \Big(\xi+Z_d, F(\xi+Z_d) - F_d \Big)$ is $\kappa$-like. $\qedsymbol$

\textbf{Proof of Lemma \ref{lemma:pdresulta}}: These can be obtained by inspection of the expressions of $X,Y,Z$ and $\zc$. Let
\begin{eqnarray}
    &&T_{\mathcal{P}}^{\mathcal{U}}(x,y,z,\psi) \triangleq \left[\begin{array}{c;{2pt/2pt}c}  R_{\mathcal{P}}^{\mathcal{W}}R_{\mathcal{W}}^{\mathcal{U}}(\psi) & \begin{array}{c} x \\ y \\ z \end{array} \\ \hdashline[2pt/2pt] \begin{array}{ccc} 0 & 0 & 0 \end{array} & 1 \end{array}\right] \nonumber \\
    &&T_{\mathcal{U}}^{\mathcal{E}}(q_m) \triangleq \left[\begin{array}{c;{2pt/2pt}c}  
    R_{\mathcal{U}}^{\mathcal{E}}(q_m) & p_{\mathcal{U}}^{\mathcal{E}}(q_m) \\ \hdashline[2pt/2pt] \begin{array}{ccc} 0 & 0 & 0 \end{array} & 1 \end{array}\right]
\end{eqnarray}

\noindent in which $R_{\mathcal{P}}^{\mathcal{W}}$, $R_{\mathcal{W}}^{\mathcal{U}}(\psi)$ were defined in Section \ref{sec:psa}. Note that, since $\theta, \phi$ are considered constant in this paper (Assumption \ref{assum:slowvar}), only the dependence on $\psi$ is made explicit. Furthermore, $(R_{\mathcal{U}}^{\mathcal{E}}(q_m), p_{\mathcal{U}}^{\mathcal{E}}(q_m))$ can be obtained using standard forward kinematics for manipulators. Then:
\begin{equation}
\label{eq:Trel}
    T_{\mathcal{P}}^{\mathcal{U}}(x,y,z,\psi)T_{\mathcal{U}}^{\mathcal{E}}(q_m) = \left[\begin{array}{ccc;{2pt/2pt}c}  \widecheck{x}(q) \hspace{-6pt} & \widecheck{y}(q)  \hspace{-6pt} & \zc(q)  \hspace{-2pt} &  \hspace{-6pt} \begin{array}{c} X(q) \\ Y(q) \\ Z(q) \end{array} \hspace{-6pt} \\ \hdashline[2pt/2pt] 0 & 0 & 0 & 1 \end{array}\right].
\end{equation}
From these, expressions for the variables $X,Y,Z$ and $\zc$ can be obtained, and the relations in Proposition \ref{lemma:pdresulta} can be easily checked by inspection. $\qedsymbol$

\textbf{Proof of Lemma \ref{lemma:pod}}: It is clear that $\|\dot{r}_Z\|^2 + \dot{q}^{\top} E \dot{q} \geq 0$ since $E$ is a positive semidefinite matrix (see Definition \ref{def:W}). It remains to prove that it is zero only if $\dot{q}=0$.

Note that it is zero if and only if   $\dot{q}^\top E \dot{q}=0$ and $\|\dot{r}_Z\|^2=\dot{q}^{\top} \nabla_q r_Z \nabla_q r_Z^\top \dot{q} = 0$, since these two terms are obviously nonnegative. However, $\dot{q}^\top E \dot{q}=0$ if and only if all $\dot{q}_i=0$ for all $i \not=3$. Substituting this information into $\|\dot{r}_Z\|^2 = 0$, one obtains $\left(\frac{\partial r_Z}{\partial z}\dot{z}\right)^2=0$. But from Lemma  \ref{lemma:pdresulta}, $\frac{\partial r_Z}{\partial z} = \frac{\partial Z}{\partial z} = 1$, and thus it follows that $\dot{z}=\dot{q}_3 = 0$. $\qedsymbol$

\textbf{Proof of Lemma \ref{lemma:partialz}}: For \eqref{eq:roteqs}, it is necessary to use the equation in kinematics that dictates that for any axis vector $\widecheck{g}(q)$, $\frac{\partial}{\partial q_i} \widecheck{g}(q) = \widecheck{r}_i(q) \times \widecheck{g}(q)$, in which $\widecheck{r}_i(q)$ is the axis vector for the rotation caused by varying only the configuration $q_i$. For $q_i = \psi$, this axis is $\az$, and for $q_i = q_{m,j}$, this axis is $\widecheck{n}_j(q)$.

For \eqref{eq:roteqs2}, the definition $r_O(q) \triangleq 1 + \ez^{\top} \zc(q)$ is used, and thus $\frac{\partial r_O}{\partial q_i}(q) = \ez^{\top} \frac{\partial \zc(q)}{\partial q_i}$. The desired results follow from \eqref{eq:roteqs}, the triple product permutation equalities $a^\top (b \times c) = c^\top (a \times b) = b^\top (c \times a)$ and Definition \ref{def:xyz}. $\qedsymbol$

\textbf{Proof of Lemma \ref{lemma:skkt}:} Suppose $u_H(q)=0$. To see which conditions entail $u_H(q)=0$, the KKT conditions in \eqref{eq:kkt}, that are necessary and sufficient for strictly convex QP problems \cite{boyd}, are analyzed. These conditions for the optimization problem in \eqref{eq:uH} with the optimal solution $u_H(q) = \mu=0$ imply that:
\begin{eqnarray}
    && \mbox{Stationarity}:  \nonumber \\
    && \nabla_q r_Z \kappa_F {-} \nabla_q B \lambda - \underline{\lambda} + \overline{\lambda} = 0. \nonumber \\
    && \mbox{Complementarity}:  \nonumber \\
    &&  \kappa_B(B)\lambda = 0 \ ,  \ \underline{b}_i \underline{\lambda}_i = 0 \ , \  \overline{b}_i \overline{\lambda}_i = 0\nonumber \\
    && \mbox{Primal feasibility}:  \nonumber \\
    && 0 \geq -\kappa_B (B) \ , \ \underline{b} \leq 0 \leq \overline{b}  \nonumber \\
    && \mbox{Dual feasibility}:  \nonumber \\
    && \lambda \geq 0 \ , \ \underline{\lambda} \geq 0 \ , \ \overline{\lambda} \geq 0 \nonumber
\end{eqnarray}
\noindent in which $\kappa_F = \frac{\partial }{\partial r_Z}V_F = \kappa_F(Z,F(Z)-F_d)$, the dependence on $q$ was omitted, the indexing $i$ refers to the components of the configuration $q$ and the variables $\lambda \in \mathbb{R}, \underline{\lambda}, \overline{\lambda} \in \mathbb{R}^n$ are the dual variables for the CBF and lower/upper bound velocities constraints, respectively. 

These conditions will be simplified into an \emph{equivalent form}. This procedure starts by considering only the $x$, $y$, $z$ and $\psi$ components of these conditions. From the ``Complementarity'' conditions, and the fact that for these entries $\underline{b}_i$ and $\overline{b}_i$ are nonzero constants (see \eqref{eq:bdef}), it is possible to conclude that for these four entries $\underline{\lambda}_i = \overline{\lambda}_i = 0$. Using this information and analyzing the ``Stationarity'' condition for $z$:
\begin{equation}
\label{eq:lambdac}
    \underbrace{\frac{\partial r_Z}{\partial z}}_{\frac{\partial }{\partial z}Z = 1}  \kappa_F - \underbrace{\frac{\partial B}{\partial z}}_{1}\lambda = 0 \ \rightarrow \ \lambda =  \kappa_F
\end{equation}
\noindent in which the relations in Lemma \ref{lemma:pdresulta} and the definition of $B$ were used. Using the ``Stationarity'' condition for $x,y,\psi$:
\[
      \lambda \frac{\partial B}{\partial x} = \frac{\partial r_Z}{\partial x} \kappa_F  \  ,  \    \lambda \frac{\partial B}{\partial y} = \frac{\partial r_Z}{\partial y} \kappa_F  \  ,  \    \lambda \frac{\partial B}{\partial \psi} = \frac{\partial r_Z}{\partial \psi}\kappa_F.
\]

The last relation, for $\psi$, can be further simplified by noting that $B = Z - Z_d^* - \kappa_{A}(A)$ and thus the terms  $\frac{\partial r_Z}{\partial \psi}   \kappa_F = \frac{\partial Z}{\partial \psi}   \kappa_F$ cancel out. Furthermore, using $\frac{\partial r_Z}{\partial x}  = \frac{\partial  Z}{\partial x} = 0$, $\frac{\partial r_Z }{\partial y} = \frac{\partial Z}{\partial y} = 0$ obtained from Lemma \ref{lemma:pdresulta} and the fact that $\lambda = \kappa_F$ from \eqref{eq:lambdac}, it is possible to obtain:
\small
\begin{equation}
\label{eq:impconclusion}
     \kappa_F \frac{\partial }{\partial x} \kappa_{A}(A) = 0 ,~\kappa_F \frac{\partial }{\partial y} \kappa_{A}(A) = 0,~\kappa_F \frac{\partial }{\partial \psi} \kappa_{A}(A) = 0.
\end{equation}
\normalsize

For the ``Stationarity'' conditions for the manipulator configurations $q_{m,j}$:
\begin{eqnarray}
\label{eq:impconclusion2}
    && \frac{\partial r_Z}{\partial q_{m,j}} \kappa_F - \frac{\partial B}{\partial q_{m,j}}\underbrace{\lambda}_{\kappa_F} - \underline{\lambda}_{m,j} + \overline{\lambda}_{m,j} = 0 \ \rightarrow \nonumber \\
    &&  \frac{\partial r_Z}{\partial q_{m,j}} \kappa_F - \left( \frac{\partial Z}{\partial q_{m,j}}{-}\frac{\partial}{\partial q_{m,j}} \kappa_A(A) \right)\kappa_F =  \underline{\lambda}_{m,j}{-}\overline{\lambda}_{m,j} \nonumber \\
    && \rightarrow \kappa_F \frac{\partial }{\partial q_{m,j}} \kappa_{A}(A) =  \underline{\lambda}_{m,j}{-}\overline{\lambda}_{m,j}
\end{eqnarray}
\noindent in which $\frac{\partial r_Z}{\partial q_{m,j}} = \frac{\partial Z}{\partial q_{m,j}}$ was used. Finally,  from the ``Primal Feasibility'' $\underline{b}_i(q) \leq 0 \leq \overline{b}_i(q)$, by looking at the conditions  for the manipulator configurations $q_{m,j}$, it is possible to conclude that $\underline{q}_{m,j} \leq q_{m,j} \leq \overline{q}_{m,j}$.
With these developments, \eqref{eq:skkt} can be derived. $\qedsymbol$

\newpage
\bibliographystyle{IEEEtran}
\bibliography{Force24}
\end{document}